\documentclass{article}

\PassOptionsToPackage{round}{natbib}
\usepackage[final]{attention_inayn}
\usepackage{float}
\usepackage{amsmath}

\usepackage[utf8]{inputenc}
\usepackage[T1]{fontenc}
\usepackage{natbib}
\usepackage{hyperref}
\usepackage{cleveref}
\usepackage{url}
\usepackage{booktabs}
\usepackage{amsfonts}
\usepackage{nicefrac}
\usepackage{microtype}
\usepackage{xcolor}
\usepackage{graphicx}
\usepackage{tikz}
\usepackage{listings}
\usepackage{tcolorbox}
\usepackage{algorithm}
\usepackage{algpseudocode}
\usepackage{amsthm}
\newtheorem{theorem}{Theorem}
\newtheorem{lemma}{Lemma}
\newtheorem{remark}{Remark}
\usepackage{caption}
\usepackage{tabularx}
\usepackage{bytefield}

\usepackage{pgfplots}
\usepackage{pgfplotstable}
\usepgfplotslibrary{groupplots}
\pgfplotsset{compat=1.17}

\usetikzlibrary{arrows.meta,positioning,shapes.geometric,shapes.symbols,fit,backgrounds,calc}

\definecolor{highlightbg}{RGB}{245,245,245}

\title{
  Recurrence-Complete Frame-based Action Models \\
  \vspace{0.5em}
  \normalsize{\it Long-horizon perception requires rethinking recurrence}
}

\author{
  Michael Keiblinger \\
  Prime Intellect \\
  \href{mailto:mike@primeintellect.ai}{mike@primeintellect.ai}
}

\begin{document}

{
\begingroup
\begin{figure*}
    \centering
    \includegraphics[width=0.125\textwidth]{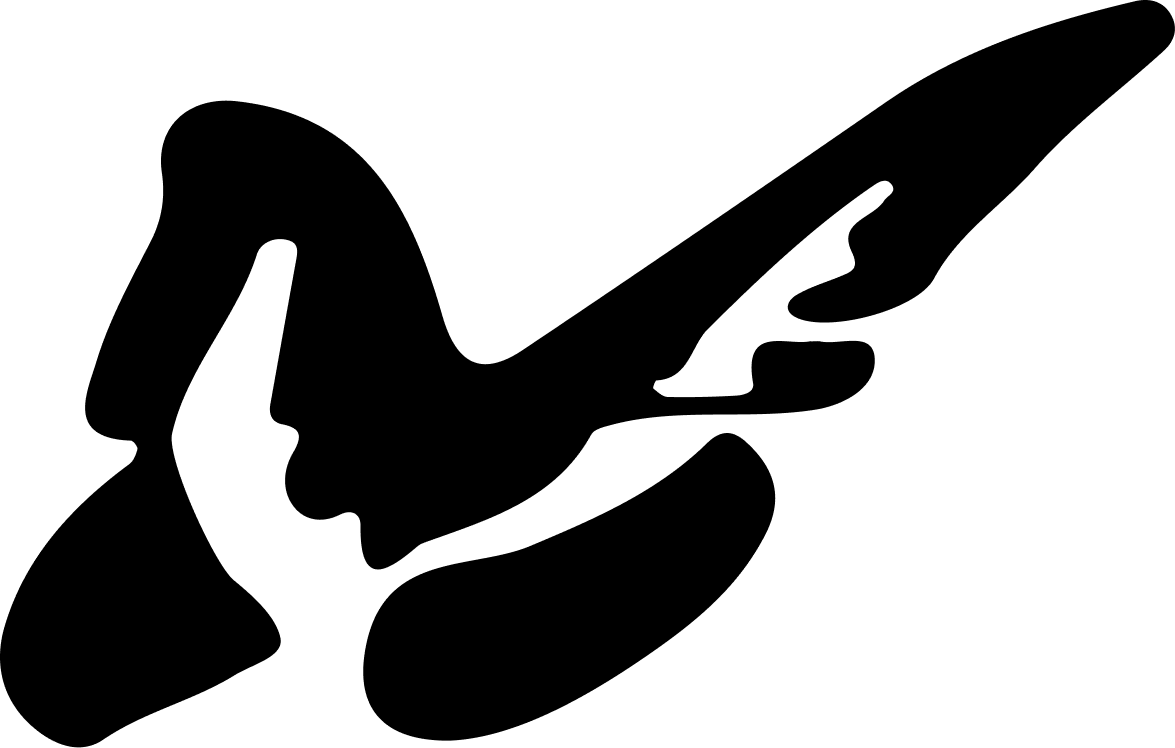}
\end{figure*}
\endgroup
}
\setcounter{figure}{0}

\maketitle
\begin{abstract}
In recent years, attention-like mechanisms have been used to great success in the space of large language models, unlocking scaling potential to a previously unthinkable extent.
``Attention Is All You Need'' famously claims RNN cells are not needed in conjunction with attention. We challenge this view.
In this paper, we point to existing proofs that architectures with fully parallelizable forward or backward passes cannot represent classes of problems specifically interesting for long-running agentic tasks.
We further conjecture a critical time $t$ beyond which non-recurrence-complete models fail to aggregate inputs correctly, with concrete implications for agentic systems (e.g., software engineering agents).
To address this, we introduce a recurrence-complete architecture and train it on GitHub-derived action sequences. Loss follows a power law in the trained sequence length while the parameter count remains fixed. Moreover, longer-sequence training always amortizes its linearly increasing wall-time cost, yielding lower loss as a function of wall time.
\end{abstract}
\newpage

\tableofcontents

\newpage
\section{Introduction}
\label{sec:introduction}

Large language models built around attention have transformed sequence modeling, enabling unprecedented scale and broad competence across text, code, and multimodal inputs. This success has motivated architectures that further emphasize parallelism over time, including state--space models and ``parallelizable RNNs'', which trade strict hidden-state dependencies for scan-style aggregation. A natural reading of this trajectory is that recurrence---in the strict sense of computation that \emph{must} proceed serially---is no longer essential. In this paper, we argue the opposite: for long-horizon perception and agentic control, there is a class of tasks for which \emph{true} serial computation is not optional but required, and that any architecture whose forward or backward passes are fully parallelizable cannot, in general, represent the needed computations.

Our argument centers on two notions that we make precise in Sections~\ref{sec:depth-as-a-function-of-sequence-length} and~\ref{sec:input-data-aggregation}. First, we define \emph{true depth} as the number of inherently sequential (non-parallelizable) operations in the computation trace of a model. Second, following \citet{zhang2024autoregressivechainthought}, we say an architecture is \emph{recurrence-complete} if it can realize recurrent updates of the form
\(
\mathbf{h}_t = g(\mathbf{h}_{t-1}, \mathbf{h}_{t-2}, \ldots, \mathbf{h}_{t-k}, \mathbf{x}_t)
\)
for general (including non-associative) $g$. Under finite/constant precision and a constant number of layers, time-parallel architectures such as Transformers instantiate constant-depth circuit families; prior work placed such families in $\mathrm{TC}^0$ and, under stronger assumptions, in $\mathrm{AC}^0$ \citep{merrill2025illusionstatestatespacemodels,li2024chainthoughtempowerstransformers}. In contrast, strict, hidden-state-dependent recurrences necessarily exhibit true depth $\Omega(n)$ in sequence length $n$.

From these premises we derive three consequences (proofs in App.~\ref{sec:no-free-lunch-proof}--\ref{sec:parallelizable-input-data-aggregation-precludes-recurrence-completeness}). (i) A model with a parallelizable forward \emph{or} backward pass cannot be recurrence-complete (a ``No Free Lunch for Parallelism'' \cite{zhang2024autoregressivechainthought}). (ii) Architectures with parallelizable input aggregation (prefix-scan-like reductions) also cannot be recurrence-complete. (iii) Consequently, families such as Mamba, S4, RWKV, Min-LSTM/GRU, and constant-layer Transformers do not, in general, possess the serial computational depth required for worst-case long-horizon problems \citep{gu2024mambalineartimesequencemodeling,gu2022efficientlymodelinglongsequences,peng2023rwkvreinventingrnnstransformer,feng2024rnnsneeded,beck2024xlstmextendedlongshortterm}.

We then identify a task property that makes these limits operational: \emph{input-length proportionality}. In such problems, correctly aggregating observations up to time $t$ requires $\,\Theta(t)\,$ truly sequential steps---no parallel reordering or associative scan eliminates the dependency chain. We formalize a related failure mode for time-parallel models: \emph{input aggregation criticality}, the maximal length beyond which a non-recurrence-complete model can no longer form the correct latent state due to bounded true depth per layer stack. As the ratio $n/L$ (sequence length over layer count) grows, we predict a degradation in representational fidelity even if attention can, in principle, attend to all tokens.

To make these ideas testable, we design synthetic diagnostics with explicit, data-dependent control flow. The \emph{Forward-Referencing Jumps Task} (FRJT) forces strictly serial evaluation of a straight-line program with forward jumps; whether an instruction executes cannot be known until the previous instruction resolves. A second benchmark, the \emph{Maze Position Tracking Task}, introduces withheld transitions that require state reconstruction rather than simple parallel counting. Across both, time-parallel models exhibit accuracy cliffs as depth grows, while a 1-layer LSTM---which is strictly serial---generalizes substantially farther (see section \ref{sec:experiments} for experiments). These results are consistent with our theory: when the underlying computation is non-scannable, architectures without depth that scales with sequence length falter.

We next connect these diagnostics to practical long-horizon perception. Agentic systems that interact with tools (shells, editors, browsers) observe streams that encode only partial state with frequent side effects. Many relevant variables are \emph{latent} and only inferable via persistent, serial bookkeeping (e.g., incremental diffs, file system mutations, UI cursor state). To study this regime at scale, we introduce a \emph{frame-based action modeling} setting: each time step provides a \emph{frame}---a complete, fixed-size 2D view of the interface (e.g., a terminal grid)---paired with the next low-level action (keystrokes or control sequences). We compile such data automatically from Git histories by reconstructing plausible editor sessions and capturing the resulting terminal frame buffer with a compact, lossless \texttt{termstreamxz} format (Figures~\ref{fig:git-history}--\ref{fig:terminal-frame}). The result is ``text-video with actions'', a natural substrate for long-horizon sequence learning.

Motivated by the above, we propose a \emph{Recurrence-Complete Frame-based Action Model}. Each frame is embedded by a transformer head with intra-frame pooling, but temporal integration is performed by a residual stack of LSTM cells, deliberately embracing non-parallelizable serial computation. We train with full backpropagation through time using a streaming, recompute-on-the-fly schedule that keeps activation memory effectively $O(1)$ in sequence length (at the cost of wall-time), aligning compute with the serial nature of the problem.

Our central empirical finding is a robust \emph{power law in trained sequence length at fixed parameter count} (see \ref{sec:fbam-experiments}). On GitHub-derived action sequences, increasing the number of frames per example monotonically lowers loss at a fixed step budget according to $\mathrm{loss}(L\mid s)\approx A(s)L^{-\alpha(s)}$, with $\alpha(s)$ rising early in training and saturating later. When accounting for the linear wall-time cost of longer sequences, the extra expense is \emph{amortized}: beyond a crossover, longer-sequence runs overtake shorter ones on loss vs.\ wall time and maintain an advantage thereafter. Importantly, unlike standard language modeling where longer contexts often chiefly improve late tokens, we observe uniform improvement across early and late positions, indicating genuine enhancement of the model's \emph{perceptual} state rather than opportunistic use of extra context.

\paragraph{Contributions.}
\begin{enumerate}
  \item \textbf{Theory.} We formalize \emph{recurrence completeness} and \emph{true depth}, and prove that architectures with parallelizable forward/backward passes or parallelizable input aggregation cannot be recurrence-complete (App.~\ref{sec:no-free-lunch-proof}--\ref{sec:parallelizable-input-data-aggregation-precludes-recurrence-completeness}). We introduce \emph{input-length proportionality} and \emph{input aggregation criticality} as operational diagnostics.
  \item \textbf{Diagnostics.} We propose FRJT and the Withheld Maze Position Tracking benchmarks that force serial computation. Under matched budgets, time-parallel models (Transformers, Mamba) exhibit depth-dependent breakdowns; a lightweight LSTM maintains performance to significantly greater depths.
  \item \textbf{Model and data.} We introduce a \emph{frame-based action} formulation and a recurrence-complete architecture that integrates a transformer frame head with an LSTM temporal backbone. We construct large-scale training corpora from Git histories by rendering editor sessions into terminal frames with action logs.
  \item \textbf{Scaling results.} Holding parameters fixed, loss follows a power law in sequence length; the longer-sequence runs ultimately dominate on loss vs.\ wall time. We provide measurements of the evolving exponent $\alpha(s)$ and discuss implications for optimization and hardware efficiency.
\end{enumerate}

\paragraph{Scope and implications.}
Our claims are not that attention is ineffective---rather, we identify a broad class of long-horizon, side-effect-laden tasks where non-scannable dependencies arise and where \emph{some} non-parallelizable computation is indispensable, merely challenging the notion that attention is \emph{all} you need. In such regimes, serial integration (e.g., LSTMs with Constant Error Carousel) may be a necessary complement to attention.
Additionally we note that deep residual networks can be viewed as unrolled gated recurrences \citep{hochreiter1997long,srivastava2015highwaynetworks,he2015deepresiduallearningimage,Schmidhuber2025WhoInventedRes}, contextualizing our scaling results with respect to sequence length by means of RNNs acting more akin to ``virtual layers'', rather than sequence models per se.

This perspective also caveats recent chain-of-thought results: textual scratchpads can externalize state, but perception at the decision point still hinges on correctly \emph{aggregating} long streams (section \ref{sec:nuance-cot}).

\newpage
\section{Depth as a function of sequence length}
\label{sec:depth-as-a-function-of-sequence-length}
For clarity we define ``true depth'' as the number of truly sequential operations that can be performed by the model.
Truly sequential operations are operations that are not parallelizable, i.e., operations that depend on the output of previous operations.
Additionally, these operations should not simplify to a smaller number of operations, i.e., they are separated by nonlinearities. Formally, true depth is the length of the longest directed path in the computation DAG (unit-cost primitive gates, including elementwise nonlinearities and non-associative mixing ops).

Any occurrence of ``depth'' in this paper shall refer to ``true depth''.

``Depth as a function of sequence length'' is a term we use to distinguish from cases where true depth does not grow as fast as $n$, where $n$ is the sequence length.

Transformers with a constant number of layers (true depth $O(1)$) and finite/constant-precision arithmetic form DLOGTIME-uniform constant-depth circuit families. Under these assumptions, prior work placed them in $\mathrm{TC}^0$ \citep{merrill2025illusionstatestatespacemodels},
and later work refined the upper bound to $\mathrm{AC}^0$ \citep{li2024chainthoughtempowerstransformers} under constant-bit precision for the activations/softmax.

This is distinctly different from the case of ``depth as a function of sequence length''.
Traditional RNN cells are ``hidden-state dependent'', strictly non-parallelizable, and thus have a depth of $O(n)$.
In recent years, numerous ``parallelizable RNNs'' have been proposed that do not possess the property of ``depth as a function of sequence length''; examples include the Min-LSTM and Min-GRU cells, which explicitly remove the hidden-state dependency of LSTM and GRU cells, arriving at equations similar to state-space models \cite{feng2024rnnsneeded}.

To separate the concept of ``repetition'' from true recurrence, we additionally define the term ``recurrence complete'', analogous to \cite{zhang2024autoregressivechainthought}:
\begin{quote}
  \itshape
  \begin{equation}
    \mathbf{h}_t= f(\mathbf{x}_t) = g(\mathbf{h}_{t-1}, \mathbf{h}_{t-2}, \mathbf{h}_{t-3}, \cdots, \mathbf{h}_{t-k})
    \label{eq:recur}
  \end{equation}
  A model is said to be recurrence-complete if it can represent any recurrent function as specified in Equation \ref{eq:recur}.
  \hfill—\citeauthor{zhang2024autoregressivechainthought}, \emph{\citeyear{zhang2024autoregressivechainthought}}
\end{quote}
We note for clarity that $g(x)$ here can be any general function, including non-associative functions.
Additionally, it can be trivially shown that any model with a parallelizable backward pass cannot be recurrence-complete\footnote{See Appendix \ref{sec:no-free-lunch-proof} \& \ref{sec:reverse-depth-nonparallelizable} for proofs.}:
\begin{quote}
  We propose a ``No Free Lunch'' rule for parallel computing in neural models: parallel training is a must trade-off for Recurrent-Completeness, and both cannot be achieved simultaneously. Specifically, a true recurrent (RC) model cannot be parallelized during either inference or training, as the computation of $\mathbf{h}_{t+1}$ strictly depends on $\mathbf{h}_{t}$ in a sequential manner.
  This can be proven by contradiction. Assume a true \textbf{recurrent} model can be trained or inferred in parallel. Then the acquisition of $\mathbf{h}_{t+1}$ can occur at the same time as $\mathbf{h}_{t}$, meaning that $\mathbf{h}_{t}$ is not a necessary dependency for $\mathbf{h}_{t+1}$. This implies that $\mathbf{h}_{t+1}$ could be computed using some other variable, say $\mathbf{v}$, which is independent of $\mathbf{h}_{t}$. Consequently, this model would not be \textbf{recurrent}, as $\mathbf{h}_{t+1}$ can be expressed as a function of solely $\mathbf{v}$, $g(\mathbf{v})$, contradicting our initial assumption of the model being recurrent.

  \hfill—\citeauthor{zhang2024autoregressivechainthought}, \emph{\citeyear{zhang2024autoregressivechainthought}}
\end{quote}

This coincides with the definition of ``Parallel Sequential Duality'' as defined in \cite{yau2025sequentialparalleldualityprefixscannable}.

We can thus conclude that a non-recurrence-complete model does not possess the property ``depth as a function of sequence length''.

The term ``depth as a function of sequence length'' is a subset of ``recurrence completeness'' where $g(x)$ is a universal approximator assuming sufficient hidden-state capacity for the task at hand.
Specifically, under the standard capacity conditions of an unbounded hidden state and a universal transition function $g(x)$, any architecture whose true-sequential depth grows with the sequence length is recurrence-complete.
Since recurrence-completeness in turn forces $\Omega(n)$ serial steps, the two notions coincide under these assumptions.
\pagebreak

The following is a non-exhaustive list of architectures that are \textbf{NOT} recurrence-complete:
\begin{itemize}
  \item \textbf{Transformers} \citep{vaswani2017attention} \emph{and efficient-attention variants}: Linear Transformers \citep{katharopoulos2020transformersrnnsfastautoregressive}, Performer \citep{choromanski2022rethinkingattentionperformers}, Linformer \citep{wang2020linformerselfattentionlinearcomplexity}, Longformer \citep{beltagy2020longformerlongdocumenttransformer}, BigBird \citep{zaheer2021bigbirdtransformerslonger}, Reformer \citep{kitaev2020reformerefficienttransformer}.
  \item \textbf{State-space / scan-style families}: S4 \citep{gu2022efficientlymodelinglongsequences}, diagonal/low-rank SSMs (S4D/DSS) \citep{gupta2022diagonalstatespaceseffective}, S5 \citep{smith2023simplifiedstatespacelayers}, H3 \citep{fu2023hungryhungryhipposlanguage}, Mamba and Mamba-2 (SSD) \citep{gu2024mambalineartimesequencemodeling,dao2024transformersssmsgeneralizedmodels}.
  \item \textbf{Retention-based architectures}: RetNet (Retentive Networks) \citep{sun2023retentivenetworksuccessortransformer}.
  \item \textbf{Gated-linear / delta-rule variants}: Gated Linear Attention (GLA) \citep{yang2024gatedlinearattentiontransformers}, DeltaNet \citep{yang2024parallelizinglineartransformersdelta}.
  \item \textbf{Convolutional / token-mixing families}: Hyena \citep{poli2023hyenahierarchylargerconvolutional}, Monarch Mixer \citep{fu2023monarchmixersimplesubquadratic,fu2023flashfftconvefficientconvolutionslong}, Temporal Convolutional Networks (TCN) \citep{bai2018empiricalevaluationgenericconvolutional}, Gated CNNs \citep{dauphin2017languagemodelinggatedconvolutional}.
  \item \textbf{Parallelizable “Min” RNNs}: Min-LSTM and Min-GRU \citep{feng2024rnnsneeded}.
  \item \textbf{RWKV} \citep{peng2023rwkvreinventingrnnstransformer}.
  \item \textbf{mLSTM} (component of the xLSTM architecture) \citep{beck2024xlstmextendedlongshortterm}.
\end{itemize}

\section{Input Aggregation}
\label{sec:input-data-aggregation}

We define the term ``input aggregation'' as follows:

To compute an output $y_t$ from a sequence of data $x_t$ for $t=1,2,3,\cdots,n$, any sequence model must consider all values ${x_i | i \in [1,t]}$ to form a latent representation from which the final prediction $y_t$ can be computed.
Aggregation is thus the process of compressing the sequence of data into a latent representation of constant size, independent of $n$.

\begin{eqnarray}
  \mathbf{h}_t &=& f(\mathbf{x}_1, \mathbf{x}_2, \cdots, \mathbf{x}_t) \\
  y_t &=& g(\mathbf{h}_t)
\end{eqnarray}

Notably, aggregation of input data can be performed in a parallelizable manner.
As long as the computation of $\mathbf{h}_{t+1}$ does not depend on $\mathbf{h}_t$, the computation of $\mathbf{h}_t$ can be parallelized.

However, in a special case of what we defined as ``input aggregation'', the computation of $\mathbf{h}_t$ does depend on $\mathbf{h}_{t-1}$.
\begin{eqnarray}
  \mathbf{h}_1 = f(\mathbf{x}_1),
  \quad
  \mathbf{h}_t = f(\mathbf{h}_{t-1},\,\mathbf{x}_t)
  \quad (t > 1)
\end{eqnarray}

Once again it can be trivially shown that any model with parallelizable input aggregation cannot be recurrence-complete\footnote{See Appendix for proof \ref{sec:parallelizable-input-data-aggregation-precludes-recurrence-completeness}}.

\section{Input-length proportionality}
We will now define a colloquial term under which input aggregation implies sequentially applying a transition function some number of times that is proportional to the input length $n$
and parallel aggregation is either not possible or brittle such that neural architectures are unlikely to learn the task at hand.

Consider the following task as an example:
Consider a sequence of N instructions, each of which modify the state of a particular variable.
The sequence of instructions is interspersed with queries about the current state of a particular variable.
The result of these queries may affect the state of the variable in subsequent instructions or be independent of it.
A language with the following properties may look as follows:
\begin{figure}
\begin{tcolorbox}[
  colback=highlightbg,
  colframe=black!50,
  arc=1mm,
  boxrule=0.4pt,
  left=2pt,
  right=2pt,
  top=2pt,
  bottom=2pt,
]
\begin{flushleft}
  \ttfamily
  \centering
  \fontsize{8pt}{8pt}\selectfont
  \color{green}
  LOAD R1 80\\
  LOAD R2 19\\
  LOAD R3 50\\
  \color{black}
  LABEL L0\\
  \color{green}
  ADD R2 R2\\
  JZ  R2 L3\\
  \begin{flushright}
    \color{black}
    $\leftarrow$ True sequential op
  \end{flushright}\vspace*{-2.0em}
  JMP L2\\
  \color{black}
  LABEL L1\\
  \color{red}
  ADD R1 R2\\
  ADD R2 R1\\
  SUB R1 R2\\
  JZ  R2 L9\\
  JMP L4\\
  \color{black}
  LABEL L2\\
  \color{green}
  ADD R3 R3\\
  \begin{flushright}
    \color{black}
    $\leftarrow$ True sequential op
  \end{flushright}\vspace*{-2.0em}
  JNZ R1 L3\\
  \color{black}
  \dots \\
  \color{black}
  LABEL L12 \\
  \color{red}
  HALT A\\
  \color{black}
  LABEL L13\\
  \color{green}
  HALT B
\end{flushleft}
\end{tcolorbox}
\caption{A sequence of instructions with strict data-dependence.}
\label{fig:strict-data-dependence}
\end{figure}
This sequence of instructions has strict data-dependence, as the control flow depends on the result of the previous instruction.
Instructions may either execute, or be skipped. To achieve strict input-length proportionality, we only allow forward referencing jumps,
which rules out any form of loops, where the amount of true sequential operations required to know the final state of the program may drastically
exceed the input length $n$.
By allowing only forward-referencing jumps, we ensure that the number of truly sequential operations required is strictly less than $n$
but still proportional to $n$ in the general case.
Whether a particular instruction executes cannot be known until the previous instruction has executed.
We mark whether a particular instruction is executed in green and skipped in red.
Whether the program halts in state A or B thus cannot be known until the entire program has been executed.
We refer to this task as the ``Forward-Referencing Jumps Task'' (FRJT).
Given that the task is input-length proportional, it can at least be evaluated in $P$.
Evaluating such a program forces strict data-dependence at all times.
Specifically, the FRJT instantiates a pointer-chasing style dependency: the identity of instruction $t{+}1$ is unknown until instruction $t$ resolves.
Pointer chasing admits $\Omega(n)$ round lower bounds in parallel models with bounded fan-in and limited random access, matching our true-depth lower bounds.

We will explore this task in more detail in the experiments section (section \ref{sec:experiments}).

\section{Input Aggregation Criticality}
\label{sec:input-agg-crit}
We define the term ``input aggregation criticality'' to refer to the maximum sequence length $n$ after which a non-recurrence-complete model can no longer correctly aggregate the input data provided, given
that the task is input-length proportional.

Specifically, aggregation criticality occurs after the number of true-sequential operations $n_{ops}$ that need to be performed to correctly aggregate the input data exceeds the constant number of true sequential operations the model can perform as a function of its layer count $L$.
This number is task- and architecture-specific.
This can be thought of as follows:
\begin{equation}
  n_{taskops} > c \cdot L
\end{equation}
where $c$ is a constant dependent on the model architecture and task.

For any solution that should generalize to arbitrary input-lengths, it is crucial that this threshold is never crossed.

In practice, there are additional constraints that the solution must not only be representable by the model, but also reachable by the training-dynamics.

\section{Relevance to agentic tasks}
For agentic tasks, the input data is typically a stream of observations from the environment that requires aggregation over long time horizons.
As long as there is some probability $p > 0$ that at some time $t$ some task-relevant information cannot be derived directly from a constant-depth transformation of all $x_i$ for $i \in [1,t]$,
input aggregation criticality will be crossed eventually.

We should note that the number of truly sequential operations of a model with parallelizable input aggregation is proportional to the layer count $L$.
As a rule of thumb, as the ratio $n/L$ increases, the quality of the model's formable $h_t$ degrades.

\subsection{Example: Environment observation}
Consider a task where an agent receives a sequence of observations $x_t$ from an environment. Each observation $x_t$ never encodes the full state of the environment, but only a partial view, i.e. editor scroll state, camera field of view, etc.
Additionally, this view may encode ``logical deltas'' instead of absolute state information.
For example, the observation may include changes to a filesystem that a coding agent is operating on, but not a full repetition of the filesystem in its current state, or a "written" record of an event that occurred in the environment.
This problem is compounded by side effects of executing commands, which may modify the state of the system in nontrivial ways without adequate reflection of said changes in the set of observations.

Given $N$ observations encoded in a sequence length of $n$, each of which can either modify or not modify the state of the system depending on the previous state of the system,
the EOP shares the data-dependence properties of the FRJT.
Given input-length proportionality, any model that performs parallelizable input aggregation will cross aggregation criticality eventually, where the quality of the formable embedding degrades rapidly.

Even if a subset of input-data can be aggregated in a parallelizable manner, if there is some probability $p > 0$ that at some point in time $t$ some task-relevant information can only be derived from a latent variable $z_t$ that is computable in $t$ true-sequential operations, the equivalence to the FRJT is still valid.

\section{An Argument from Video}
Compute-efficient time parallelism in practice implies random access into the input sequence, especially in combination with intermediates saved for backpropagation.
In many training procedures, the input sequence is stored in full on device or across the devices participating in the training.
While it may be very feasible to hold a sequence of text tokens in memory at once, this is not the case for video data.
Video data is typically heavily compressed. During decompression, usually only the current frame - and certain key frames necessary for frame interpolation - inflate to full size.
It is unwise to expect e.g. raw h264 codec bytes to be consumable by any feasible neural architecture, therefore the video data must be decompressed to serve as a training example.
For a Full HD 8-bit RGB video at 60 fps, this amounts to $60 \times 1920 \times 1080 \times 3 = 373248000$ bytes = $373.248$ MB per second of video.
To avoid memory explosion, model architectures will have to be ``streamable'' in the same way as video decoding is.
Longer training horizon should neither require more devices, nor significantly more memory beyond what is required to store the compressed video data.
This is practically achievable with left-to-right streaming recurrence.

\section{Experiments}
\label{sec:experiments}
How quickly is input aggregation criticality reached in practice?
We will start tackling this question first by evaluating the performance of a variety of architectures on synthetic tasks that are input-length proportional.

\subsection{Synthetic Tasks}

\subsubsection{Forward-Referencing Jumps Task}
In this section, we will evaluate the performance of a variety of architectures on the Forward-Referencing Jumps Task (FRJT).
For this experiment, we will generate synthetic programs with a maximum depth of $d$.
A program is said to have a depth of $d$ if it contains $d$ labels.
For each $i \in [1, d]$, 8000 programs will be generated.
This is intentional to include short programs to achieve a similar result as teacher forcing a target program state at each point in time \citep{williams1989learning}.
If the model has training signal at all points in time, it is likely to learn the correct transition function.
If a model is expected to learn the correct transition function at high depth without intermediates, learning dynamics are more likely to fail.
By mixing short and long programs, the circuits learned for the short programs will generalize to longer programs, without the need for an auxiliary objective for e.g. register state supervision.

For each block of computation, there exists a jump to a future label. Each block performs two jumps, only one of which will execute depending on a condition evaluation.
Both jumps will forward reference labels. The program halts in either state A or B depending on the final location of the program counter.
To avoid predictability, blocks occasionally jump straight to the terminating state. However, the program is more likely to jump to a future block compared to jumping to the terminating state.
Jumps are also more likely to reference blocks closer to the current block as opposed to blocks further away to maximize runtime and thus true depth.
It is intended that the probability of jumping straight to the end prematurely will accumulate over time and that it is unlikely that later parts of the program will be executed.
However, programs still have approximately 50\% code coverage.
Additionally, the chance of the program halting state being A or B respectively is approximately 50\%.

The dataset will consist of programs and their corresponding labels (Halt A or B, as determined by an interpreter).

The final layer of the model will be a binary-classification head. For time-parallel architectures, only the last point in time will be used for loss-calculation.

\begin{table}[H]
  \centering
  \caption{FRJT Transformer Performance}
  \begin{tabular}{lllll}
    \toprule
    \textbf{$n_{layer}$} & \textbf{$n_{embed}$} & \textbf{Max. Depth} & \textbf{Train Accuracy} & \textbf{Validation Accuracy} \\
    \midrule
    1 & 256 & 4 & 0.98119 & 0.65644 \\
    2 & 256 & 4 & 0.99997 & 0.74406 \\
    3 & 256 & 4 & 1.00000 & 0.74322 \\
    4 & 256 & 4 & 1.00000 & 0.76569 \\
    5 & 256 & 4 & 1.00000 & 0.79047 \\
    6 & 256 & 4 & 0.99589 & 0.81969 \\
    7 & 256 & 4 & 0.99651 & 0.78909 \\
    8 & 256 & 4 & 0.72135 & 0.67644 \\
    \midrule
    1 & 256 & 8 & 0.97504 & 0.59366 \\
    2 & 256 & 8 & 0.99445 & 0.62725 \\
    3 & 256 & 8 & 0.99283 & 0.65612 \\
    4 & 256 & 8 & 0.98863 & 0.62133 \\
    5 & 256 & 8 & 0.97039 & 0.61378 \\
    6 & 256 & 8 & 0.99278 & 0.68308 \\
    \bottomrule
  \end{tabular}
\end{table}

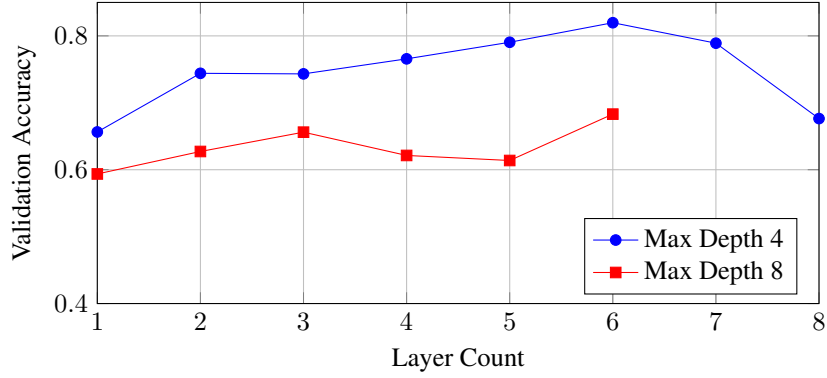
\begin{figure}[H]
    \centering
    \begin{tikzpicture}
    \begin{axis}[
        width=0.8\textwidth,
        height=0.4\textwidth,
        xlabel={Layer Count},
        ylabel={Validation Accuracy},
        xmin=1, xmax=8,
        ymin=0.4, ymax=0.85,
        xtick={1,2,3,4,5,6,7,8},
        legend pos=south east,
        grid=major,
        ]
        \addplot[mark=*,blue] coordinates {
            (1,0.65644)
            (2,0.74406)
            (3,0.74322)
            (4,0.76569)
            (5,0.79047)
            (6,0.81969)
            (7,0.78909)
            (8,0.67644)
        };
        \addplot[mark=square*,red] coordinates {
            (1,0.59366)
            (2,0.62725)
            (3,0.65612)
            (4,0.62133)
            (5,0.61378)
            (6,0.68308)
        };
        \legend{Max Depth 4, Max Depth 8}
    \end{axis}
    \end{tikzpicture}
    \caption{Transformer validation accuracy as a function of layer count for different maximum depths}
\end{figure}

\begin{table}[H]
  \centering
  \caption{FRJT Mamba Performance}
  \begin{tabular}{lllll}
    \toprule
    \textbf{$n_{layer}$} & \textbf{$n_{embed}$} & \textbf{Max. Depth} & \textbf{Train Accuracy} & \textbf{Validation Accuracy} \\
    \midrule
    1 & 256 & 4 & 0.85194 & 0.84331 \\
    2 & 256 & 4 & 1.00000 & 0.94431 \\
    3 & 256 & 4 & 1.00000 & 0.88203 \\
    4 & 256 & 4 & 1.00000 & 0.93219 \\
    5 & 256 & 4 & 1.00000 & 0.91258 \\
    6 & 256 & 4 & 1.00000 & 0.87392 \\
    7 & 256 & 4 & 0.99997 & 0.966 \\
    8 & 256 & 4 & 1.00000 & 0.98075 \\
    \midrule
    1 & 256 & 8 & 0.80129 & 0.79172 \\
    2 & 256 & 8 & 1.00000 & 0.85813 \\
    3 & 256 & 8 & 1.00000 & 0.80512 \\
    4 & 256 & 8 & 1.00000 & 0.90544 \\
    6 & 256 & 8 & 0.99672 & 0.91402 \\
    8 & 256 & 8 & 0.99486 & 0.91498 \\
    10 & 256 & 8 & 0.99999 & 0.92641 \\
    12 & 256 & 8 & 1.00000 & 0.93252 \\
    14 & 256 & 8 & 1.00000 & 0.94788 \\
    \midrule
    1 & 256 & 16 & 0.73724 & 0.71975 \\
    8 & 256 & 16 & 0.98847 & 0.83139 \\
    12 & 256 & 16 & 0.99433 & 0.85069 \\
    16 & 256 & 16 & 0.98949 & 0.87901 \\
    \bottomrule
  \end{tabular}
\end{table}

\begin{figure}[H]
    \centering
    \begin{tikzpicture}
    \begin{axis}[
        width=0.8\textwidth,
        height=0.4\textwidth,
        xlabel={Layer Count},
        ylabel={Validation Accuracy},
        xmin=1, xmax=16,
        ymin=0.7, ymax=1.0,
        xtick={1,2,3,4,5,6,7,8,9,10,11,12,13,14,15,16},
        legend pos=south east,
        grid=major,
        ]
        \addplot[mark=*,blue] coordinates {
            (1,0.84331)
            (2,0.94431)
            (3,0.88203)
            (4,0.93219)
            (5,0.91258)
            (6,0.87392)
            (7,0.966)
            (8,0.98075)
        };
        \addplot[mark=square*,red] coordinates {
            (1,0.79172)
            (2,0.85813)
            (3,0.80512)
            (4,0.90544)
            (6,0.91402)
            (8,0.91498)
            (10,0.92641)
            (12,0.93252)
            (14,0.94788)
        };
        \addplot[mark=triangle*,green] coordinates {
            (1,0.71975)
            (8,0.83139)
            (12,0.85069)
            (16,0.87901)
        };
        \legend{Max Depth 4, Max Depth 8, Max Depth 16}
    \end{axis}
    \end{tikzpicture}
    \caption{Mamba validation accuracy as a function of layer count for different maximum depths}
\end{figure}
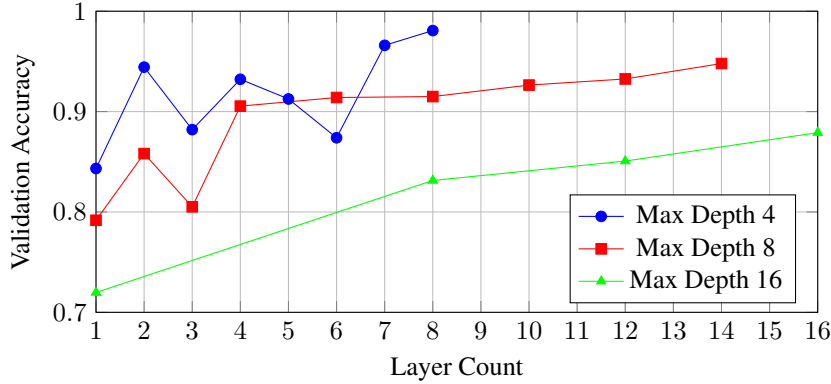

\begin{table}[H]
  \centering
  \caption{FRJT LSTM Performance}
  \begin{tabular}{lllll}
    \toprule
    \textbf{$n_{layer}$} & \textbf{$n_{embed}$} & \textbf{Max. Depth} & \textbf{Train Accuracy} & \textbf{Validation Accuracy} \\
    \midrule
    1 & 256 & 4 & 0.99905 & 0.95981 \\
    1 & 256 & 8 & 0.99119 & 0.96569 \\
    1 & 256 & 16 & 0.94628 & 0.94663 \\
    1 & 256 & 24 & 0.92969 & 0.90238 \\
    1 & 256 & 32 & 0.87409 & 0.85705 \\
    \bottomrule
  \end{tabular}
\end{table}

\begin{figure}[H]
    \centering
    \begin{tikzpicture}
    \begin{axis}[
        width=0.8\textwidth,
        height=0.4\textwidth,
        xlabel={Maximum Depth},
        ylabel={Validation Accuracy},
        xmin=4, xmax=32,
        ymin=0.8, ymax=1.0,
        xtick={4,8,16,24,32},
        legend pos=north east,
        grid=major,
        ]
        \addplot[mark=*,blue] coordinates {
            (4,0.95981)
            (8,0.96569)
            (16,0.94663)
            (24,0.90238)
            (32,0.85705)
        };
        \legend{LSTM (1 Layer)}
    \end{axis}
    \end{tikzpicture}
    \caption{LSTM validation accuracy as a function of maximum depth}
\end{figure}
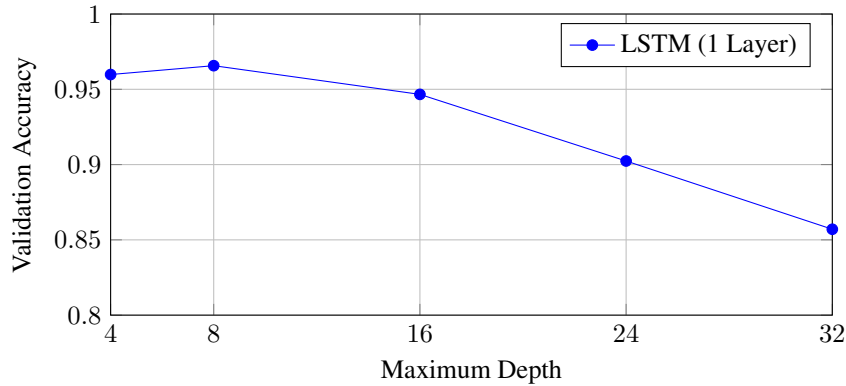

Additionally, it should be noted that every architecture tested here showed signs of overfitting (increasing validation loss after a certain point) in every run except for LSTM-runs, where
only the max. depth 32 run showed slight signs of increased validation loss after saturating at around 87\% training accuracy.

Mamba is able to learn the task up until a critical N after which layer count has to be increased proportionally.
Even the largest Mamba run with 16 layers at depth 16 scores (validation accuracy 0.87901) worse than the 1-layer LSTM (0.94663) by a substantial amount
while Mamba at 16 layers cannot be argued to be compute-efficient for the task at hand given increasing wall-time cost.

\subsubsection{Maze Position Tracking Task}
As a simpler task that is input-length proportional, we will define the ``Maze Position Tracking Task''.
The maze is a 2D grid of size $32 \times 32$.
The task is to predict the position of an agent given the sequence of movements in the maze.
The maze will be fixed for all individual examples such that the model can learn the layout of the maze as a prior.
The task is formulated as a regression task, where the model must predict the two components of the final position.
Predictions are made after each movement and thus the correct position is teacher-forced at each step.

We distinguish between two variants of the task:
\subsubsection{Unwithheld Maze Position Tracking Task}
In this variant, the model receives not only the sequence of performed movements, but also whether the movement resulted in an unchanged position.
In this variant, solving the task is as simple as counting the number of movements in each direction, excluding the movements that resulted in an unchanged position.
Even if the sequence model is replaced by a cross-temporal sum, the task can still be easily solved by the model with 100\% validation accuracy.
It can be assumed that essentially any sequence model will be able to solve this task in a fully length-generalizable fashion.
An LSTM achieves full 100\% validation accuracy and so does a transformer.

\begin{figure}[H]
  \centering
  \includegraphics[width=0.4\textwidth]{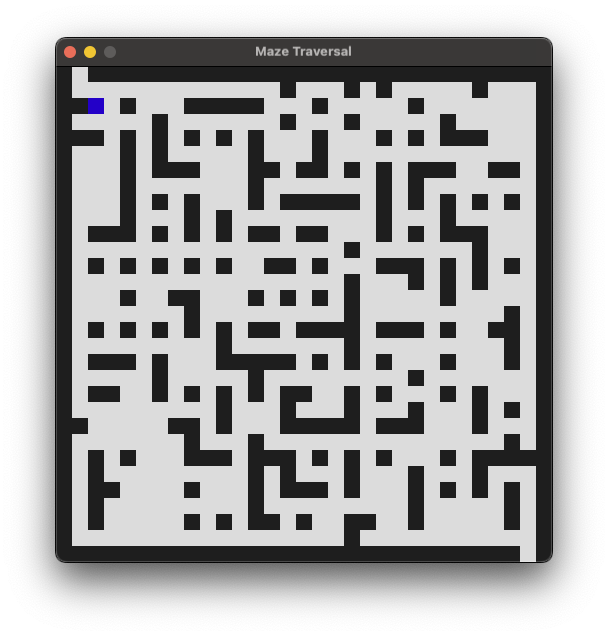}
  \caption{Visualization of the Maze Used in all Experiments}
\end{figure}

Example data of this task may look as follows:
\begin{flushleft}
  \ttfamily
  \fontsize{8pt}{8pt}\selectfont
  \color{blue}LEFT\color{black}, \color{green}LEFT\color{black} \\
  \color{blue}LEFT\color{black}, \color{red}UNCHANGED\color{black} \\
  \color{blue}RIGHT\color{black}, \color{green}RIGHT\color{black} \\
  \color{blue}UP\color{black}, \color{red}UNCHANGED\color{black} \\
  \color{blue}DOWN\color{black}, \color{green}DOWN\color{black} \\
  \color{blue}LEFT\color{black}, \color{green}LEFT\color{black} \\
  \color{blue}LEFT\color{black}, \color{red}UNCHANGED\color{black} \\
  \color{blue}RIGHT\color{black}, \color{green}RIGHT\color{black} \\
  \color{blue}RIGHT\color{black}, \color{red}UNCHANGED\color{black} \\
\end{flushleft}
However, if we sparsely withhold the resulting movement with some probability $p$, the task now requires full input-length proportional reasoning.

\subsubsection{Withheld Maze Position Tracking Task}
In this variant, the model still receives the set of performed movements, however sometimes with a probability $p$ the resulting movement is withheld.
Depth is defined as the number of movements that are withheld.
It should be noted that not each occurrence of a withheld movement is equally difficult.
A withheld movement may result in a change in position, or it may result in no change in position or it may be in a sequence of movements that cancel each other out for which it may be possible to build computational shortcuts.
\pagebreak

Example data of this task may look as follows:
\begin{flushleft}
  \ttfamily
  \fontsize{8pt}{8pt}\selectfont
  \color{blue}LEFT\color{black}, \color{green}LEFT\color{black} \\
  \color{blue}LEFT\color{black}, \color{red}UNCHANGED\color{black} \\
  \color{blue}RIGHT\color{black}, \color{green}RIGHT\color{black} \\
  \color{blue}UP\color{black}, \color{orange}WITHHELD\color{black} \\
  \color{blue}DOWN\color{black}, \color{green}DOWN\color{black} \\
  \color{blue}LEFT\color{black}, \color{green}LEFT\color{black} \\
  \color{blue}LEFT\color{black}, \color{red}UNCHANGED\color{black} \\
  \color{blue}RIGHT\color{black}, \color{red}UNCHANGED\color{black} \\
  \color{blue}RIGHT\color{black}, \color{orange}WITHHELD\color{black} \\
  \color{blue}LEFT\color{black}, \color{green}LEFT\color{black} \\
\end{flushleft}

\begin{table}[H]
  \centering
  \caption{Withheld Maze Position Tracking Task Performance}
  \begin{tabular}{llllll}
    \toprule
    \textbf{Architecture} & \textbf{$p_{withheld}$} & \textbf{$n_{layer}$} & \textbf{$n_{embed}$} & \textbf{Depth} & \textbf{Validation Accuracy} \\
    \midrule
    Sum & 20\% & 1 & 256 & 32 & 0.3785 \\
    LSTM & 20\% & 1 & 256 & 32 & 0.9942 \\
    Transformer & 20\% & 1 & 256 & 32 & 0.6105 \\
    Transformer & 20\% & 2 & 256 & 32 & 0.6388 \\
    Transformer & 20\% & 3 & 256 & 32 & 0.9020 \\
    Transformer & 20\% & 4 & 256 & 32 & 0.9321 \\
    \midrule
    Sum & 20\% & 1 & 256 & 64 & 0.2717 \\
    LSTM & 20\% & 1 & 256 & 64 & 0.9445 \\
    Transformer & 20\% & 1 & 256 & 64 & 0.4571 \\
    Transformer & 20\% & 2 & 256 & 64 & 0.6453 \\
    Transformer & 20\% & 3 & 256 & 64 & 0.7547 \\
    Transformer & 20\% & 4 & 256 & 64 & 0.8271 \\
    \bottomrule
  \end{tabular}
\end{table}

We again note for clarity that depth does not equal sequence length here, with the sequence length being significantly longer than depth, as depth is the count of withheld movements. The sequence consists of both intent and the result feedback represented as a token each.

\begin{figure}[H]
  \centering
  \includegraphics[width=0.6\textwidth]{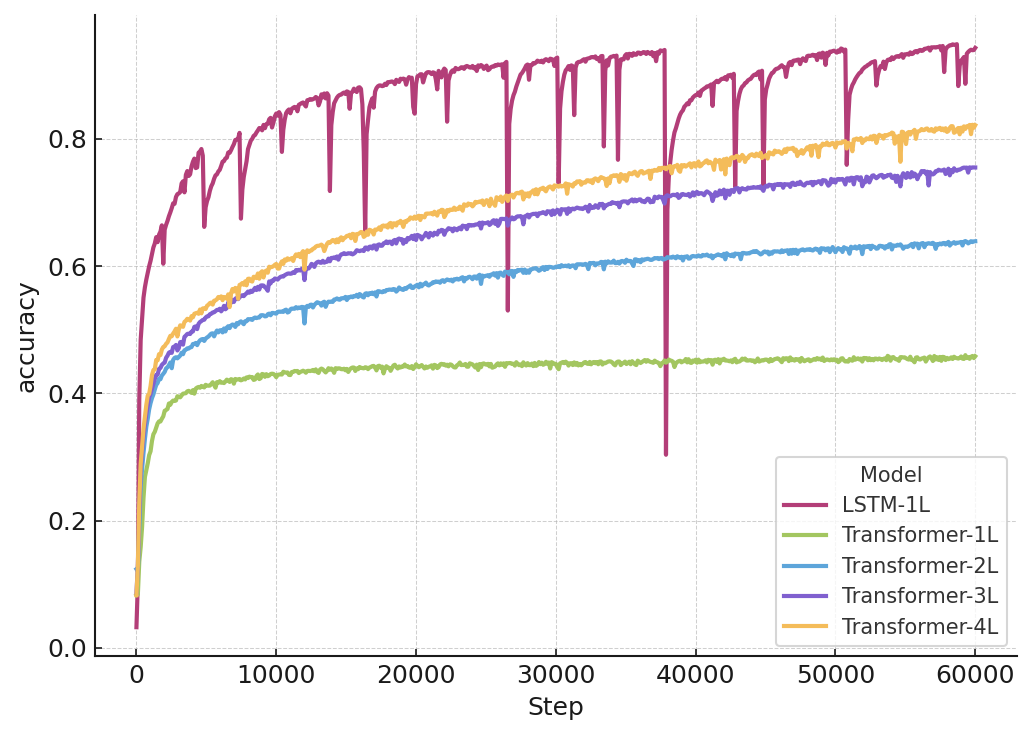}
  \caption{Validation accuracy for runs at depth 64}
\end{figure}

We note the LSTM's ``spiky'' pattern in accuracy, which coincide with loss spikes. This behavior is expected in heavily discretized objectives with sharp decision boundaries.
These spikes are followed by subsequent fast ``catch-up'', often to a higher accuracy than before. While not included in this comparison for fairness reasons, we note that the 1-Layer LSTM run - despite many spikes - continues to improve and reaches 99.17\% validation accuracy at depth 64 after 200,000 steps.

\subsection{Practical Tasks}
We will now explore a practical example of an environment observation task that is input-length proportional.

\subsubsection{Coding Agent Task}
Consider a coding agent that observes the state of a computer through a series of text-based console-commands and their respective outputs.
Assume for the sake of simplicity that every command and every file-system modification compared to the system's initial clean state is performed exclusively by the agent.
Every line of code can be assumed to have been written by the agent itself.

Depending on the exact representation of the agent's observation, modifications may either occur in the form of diff-strings,
or in the form of partially rendered text-file contents through position seeking functionality (i.e. text editor).

Suppose the agent's only view of a file is the sequence of diffs $\delta_1, \delta_2, \dots, \delta_n$ each of which ``patches'' the current file state.

Formally, the file state $S_t$ is computed as
\begin{equation}
  S_0 = \text{empty file},\quad
  S_t = \mathbf{apply}(S_{t-1},\,\delta_t)
  \quad(t=1,\dots,n)
\end{equation}

Patch representations are known to be applicable in parallel in many cases, however a clean sequence of diffs is rarely observed in practice.
In practice, executing commands induces side-effects to the state of the file system, which we will assume will not be directly observed by the agent.
Doing so would require intercepting file I/O system calls, which is not feasible in practice due to API call verbosity.
This does induce strict data-dependence into the problem because now we do need to track and materialize a sequentially consistent history of the file system to answer questions about
the command's behavior at time $t$---even if only \textit{semantically} as opposed to literal neural emulation.

Formally, we now have two stages, one being patch application given by partial observations and second being an optionally present side-effect, which accepts the state of the file system at time $t$ and returns a new state of the file system at time $t+1$.

\begin{equation}
  S_t = \mathbf{sideeffect}(\mathbf{apply}(S_{t-1},\,\delta_t))
  \quad(t=1,\dots,n)
\end{equation}

This problem in particular however can be side-stepped through a full-print of the file's contents initiated by the agent.

However, given that we are frequently ``rendering'' fragments of our code-base due to the agent's inability to aggregate its history,
we are confronted with the fact that the agent cannot possess a true representation of the code-base's trajectory.
We are left with a strictly atemporal representation of the code-base, void of true cross-temporal context as time approaches infinity.

\subsubsection{Diff-Inflate-Bench}
We propose a benchmark for language models where the model is asked to produce the final state of a file given a sequence of git diffs.
The final state is given to a judge, which is tasked with determining whether the candidate prediction is functionally equivalent to the ground truth final state.
Formatting and style differences are allowed, as well as non-semantic changes such as order of function declarations. Only semantically relevant changes are penalized.
Additionally, candidate predictions were manually inspected for surface-level intactness compared to the ground truth final state to rule out context-length induced truncation or other defects.
We chose N strictly incremental patches from the initial state. For $N=1$, the task is equivalent to a ``copy'' operation while stripping certain symbols.
To estimate the fraction of correct renderings, we obtain 48 samples consisting of the rendering prediction and the judge's verdict.
For our testing, we deliberately use patches from the tinygrad project due to its unconventional implementation approaches, high code density, high amounts of overlapping patches and optically unpredictable code \citep{tinygrad}.
We notice a concerning trend that for other repositories, performance does not degrade as rapidly as it does for tinygrad.
We believe this is due to the fact that for codebases of lower complexity, the model is able to ``guess'' the correct final state and aggregate primarily by plausibility as opposed to true state tracking.

Thinking budget is reduced to the minimum allowed by the respective model, as the benchmark specifically aims to measure the model's native ability to perceive the state of the codebase.
We explicitly acknowledge that with sufficient chain-of-thought thinking the model could externalize all state tracking needed to facilitate correct inflation of diffs, however this is not the point of the benchmark.

We evaluate on OpenAI's gpt-5-codex \& gpt-5-mini models, Google's gemini-2.5-pro and Anthropic's claude-sonnet-4.5 model.
``gpt-5-codex'' was used exclusively as the judge to ensure consistency of evaluation across models.

We observe a strict downward trend in performance as the number of diffs increases.
This implies that as the number of patches increases, the codebase becomes increasingly opaque to the model.
We note however that claude-sonnet-4.5 is able to better maintain accuracy throughout.
The task is by no means impossible to solve in parallelizable fashion, as mere patch application without side-effects is insufficient to invoke true input-length proportionality.
Additionally, the act of generating the final state of the code may help to sufficiently decompose the problem over many tokens given partial results in ways which are not necessarily comparable to the more immediate perception requirements needed during agentic workloads,
so we note that this benchmark is likely an insufficient measurement of true codebase embedding quality.
We note that given a sorted set of diffs, simply repeating non-deleted lines from the latest diff that appears to start with the current file region to be rendered would represent a parallelizable approximation
that is learnable by a transformer. Additionally, the surrounding context allows the model to ``cheat'', as it effectively leaks the state of surrounding regions at a given point in time,
eliminating the need for true state-tracking for the regions covered.

\begin{figure}[H]
  \centering
  \begin{tikzpicture}
    \definecolor{codex}{HTML}{5DA5DA}
    \definecolor{gemini}{HTML}{B33E78}
    \definecolor{mini}{HTML}{5C9863}
    \definecolor{claude}{HTML}{F4BC5A}
    \begin{axis}[
      width=0.65\linewidth, height=6.5cm,
      xlabel={Number of patches},
      ylabel={Functionally equivalent renderings (\%)},
      ymin=0, ymax=105, xmin=0, xmax=150,
      grid=both,
      major grid style={dashed,gray!40},
      minor grid style={dashed,gray!20},
      legend pos=outer north east,
      legend cell align=left,
      tick align=outside,
      tick style={black},
      every axis plot/.append style={
        semithick, mark=*, mark size=1.8pt
      },
    ]
    \addplot+[color=codex, mark options={fill=codex, draw=codex}]
    coordinates {
    (1,97.92) (2,97.92) (4,97.92) (8,89.58) (16,85.42) (24,79.17)
    (32,66.67) (48,64.58) (64,58.33) (72,52.08) (80,64.58) (88,56.25)
    (96,47.92) (104,35.42) (112,35.42) (128,29.17) (136,27.08) (148,20.83)
    }; \addlegendentry{gpt-5-codex}
    \addplot+[color=gemini, mark options={fill=gemini, draw=gemini}]
    coordinates {
    (1,97.92) (2,97.92) (4,95.83) (8,83.33) (16,85.42) (24,81.25)
    (32,70.83) (48,70.83) (64,75.00) (72,56.25) (80,72.92) (88,45.83)
    (96,56.25) (104,37.50) (112,41.67) (128,43.75) (136,47.92) (148,25.00)
    }; \addlegendentry{gemini-2.5-pro}
    \addplot+[color=mini, mark options={fill=mini, draw=mini}]
    coordinates {
    (1,100.00) (2,95.83) (4,91.67) (8,81.25) (16,60.42) (24,47.92)
    (32,35.42) (48,27.08) (64,22.92) (72,27.08) (80,12.50) (88,16.67)
    (96,6.25) (104,2.08) (112,10.42) (128,4.17) (136,10.42) (148,8.33)
    }; \addlegendentry{gpt-5-mini}
    \addplot+[color=claude, mark options={fill=claude, draw=claude}]
    coordinates {
    (1,100.00) (2,97.92) (4,100.00) (8,89.58) (16,95.74) (24,89.58)
    (32,83.33) (48,91.67) (64,91.67) (72,89.58) (80,93.62) (88,95.83)
    (96,95.83) (104,100.00) (112,84.78) (128,85.11) (136,78.72) (148,97.78)
    }; \addlegendentry{claude-4.5-sonnet}
    \end{axis}
  \end{tikzpicture}
  \caption{Diff-Bench performance}
  \label{fig:diff-bench}
\end{figure}
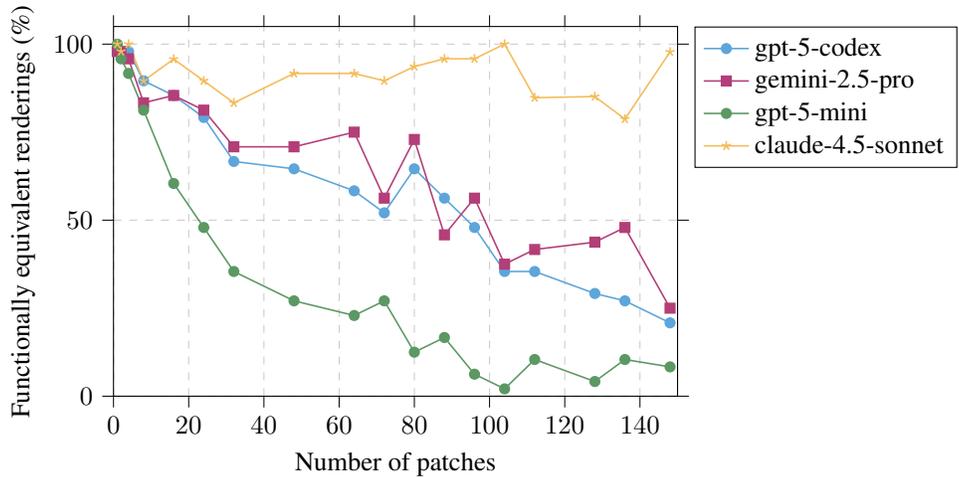

To test whether this is the approximation claude-4.5-sonnet has learned, we configure git to produce minimal diffs (U0) without redundant context, moving the task away from ``reflected in observations'' towards ``withheld, but inferable''.
The result is a collapse in accuracy across all models, including claude-4.5-sonnet.
We note however that relying on solely line-count arithmetic for inferring patch placement is potentially harsh. We therefore also evaluate diffs with exactly one line of context (U1),
allowing placement to be inferred from context while still reducing the likelihood of ``leaking too many lines''.

\begin{figure}[H]
  \centering
  \begin{tikzpicture}
    \definecolor{mini}{HTML}{5C9863}
    \definecolor{gemini}{HTML}{B33E78}
    \definecolor{claude}{HTML}{F4BC5A}
    \definecolor{codex}{HTML}{5DA5DA}
    \begin{axis}[
      width=0.65\linewidth, height=6.5cm,
      xlabel={Number of patches},
      ylabel={Functionally equivalent renderings (\%)},
      ymin=0, ymax=105, xmin=0, xmax=150,
      grid=both,
      major grid style={dashed,gray!40},
      minor grid style={dashed,gray!20},
      legend pos=outer north east,
      legend cell align=left,
      tick align=outside,
      tick style={black},
      every axis plot/.append style={semithick, mark=*, mark size=1.8pt},
    ]
    \addplot+[color=codex, mark options={fill=codex, draw=codex}]
    coordinates {
    (1,100.00) (2,89.58) (4,89.58) (8,60.42) (16,39.58) (24,20.83)
    (32,18.75) (48,14.58) (64,8.33) (72,8.33) (80,8.33) (88,6.25)
    (96,4.17) (104,4.17) (112,6.25) (128,0.00) (136,0.00) (148,0.00)
    }; \addlegendentry{gpt-5-codex}
    \addplot+[color=mini, mark options={fill=mini, draw=mini}]
    coordinates {
    (1,100.00) (2,77.08) (4,66.67) (8,29.17) (16,8.33) (24,4.17)
    (32,0.00) (48,2.08) (64,2.08) (72,0.00) (80,2.08) (88,0.00)
    (96,0.00) (104,0.00) (112,0.00) (128,0.00) (136,0.00) (148,0.00)
    }; \addlegendentry{gpt-5-mini}
    \addplot+[color=gemini, mark options={fill=gemini, draw=gemini}]
    coordinates {
    (1,100.00) (2,91.67) (4,91.67) (8,68.75) (16,50.00) (24,41.67)
    (32,18.75) (48,16.67) (64,16.67) (72,6.25) (80,10.42) (88,12.50)
    (96,0.00) (104,4.17) (112,0.00) (128,0.00) (136,2.08) (148,0.00)
    }; \addlegendentry{gemini-2.5-pro}
    \addplot+[color=claude, mark options={fill=claude, draw=claude}]
    coordinates {
    (1,100.00) (2,93.75) (4,85.42) (8,72.92) (16,58.33) (24,50.00)
    (32,35.42) (48,31.25) (64,14.58) (72,12.50) (80,31.25) (88,14.58)
    (96,8.33) (104,14.58) (112,8.33) (128,6.25) (136,8.33) (148,4.17)
    }; \addlegendentry{claude-4.5-sonnet}
    \end{axis}
  \end{tikzpicture}
  \caption{Diff-Bench performance (U0)}
\label{fig:diff-bench-u0}
\end{figure}
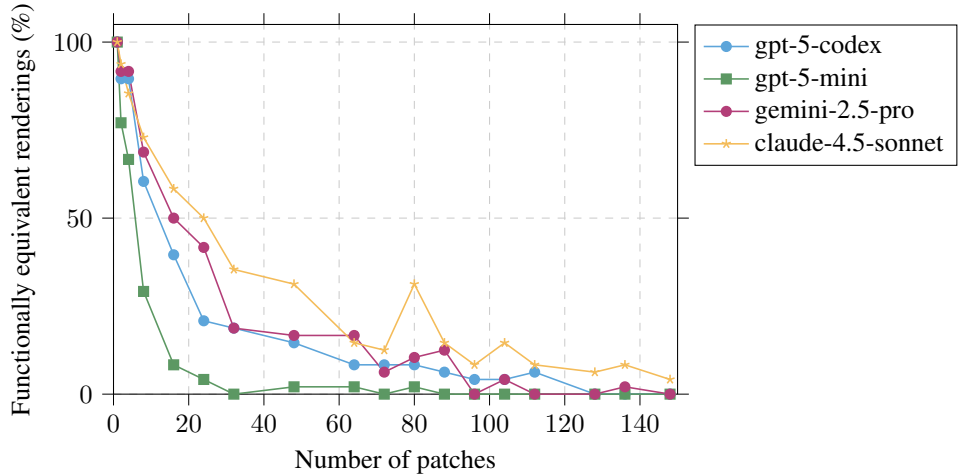

\begin{figure}[H]
  \centering
  \begin{tikzpicture}
  \definecolor{claude}{HTML}{F4BC5A}
  \definecolor{gemini}{HTML}{B33E78}
  \definecolor{codex}{HTML}{5DA5DA}
  \definecolor{mini}{HTML}{5C9863}
  \begin{axis}[
    width=0.65\linewidth, height=6.5cm,
    xlabel={Number of patches},
    ylabel={Functionally equivalent renderings (\%)},
    ymin=0, ymax=105, xmin=0, xmax=150,
    grid=both,
    major grid style={dashed,gray!40},
    minor grid style={dashed,gray!20},
    legend pos=outer north east,
    legend cell align=left,
    tick align=outside,
    tick style={black},
    every axis plot/.append style={semithick, mark=*, mark size=1.8pt},
  ]
  \addplot+[color=codex, mark options={fill=codex, draw=codex}]
  coordinates {
  (1,100.00) (2,97.92) (4,93.75) (8,89.58) (16,75.00) (24,72.92)
  (32,50.00) (48,54.17) (64,43.75) (72,33.33) (80,45.83) (88,35.42)
  (96,25.00) (104,20.83) (112,25.00) (128,16.67) (136,16.67) (148,10.42)
  }; \addlegendentry{gpt-5-codex}
  \addplot+[color=gemini, mark options={fill=gemini, draw=gemini}]
  coordinates {
  (1,97.92) (2,95.83) (4,95.83) (8,81.25) (16,81.25) (24,75.00)
  (32,66.67) (48,62.50) (64,62.50) (72,50.00) (80,41.67) (88,45.83)
  (96,50.00) (104,29.17) (112,22.92) (128,29.17) (136,22.92) (148,12.50)
  }; \addlegendentry{gemini-2.5-pro}
  \addplot+[color=mini, mark options={fill=mini, draw=mini}]
  coordinates {
  (1,100.00) (2,91.67) (4,91.67) (8,72.92) (16,41.67) (24,37.50)
  (32,16.67) (48,2.08) (64,8.33) (72,6.25) (80,6.25) (88,4.17)
  (96,0.00) (104,0.00) (112,0.00) (128,2.08) (136,0.00) (148,0.00)
  }; \addlegendentry{gpt-5-mini}
  \addplot+[color=claude, mark options={fill=claude, draw=claude}]
  coordinates {
  (1,97.92) (2,97.92) (4,100.00) (8,89.58) (16,91.67) (24,83.33)
  (32,72.92) (48,77.08) (64,79.17) (72,81.25) (80,85.42) (88,64.58)
  (96,60.42) (104,68.75) (112,68.75) (128,60.42) (136,39.58) (148,50.00)
  }; \addlegendentry{claude-4.5-sonnet}
  \end{axis}
  \end{tikzpicture}
  \caption{Diff-Bench performance (U1)}
  \label{fig:diff-bench-u1}
\end{figure}
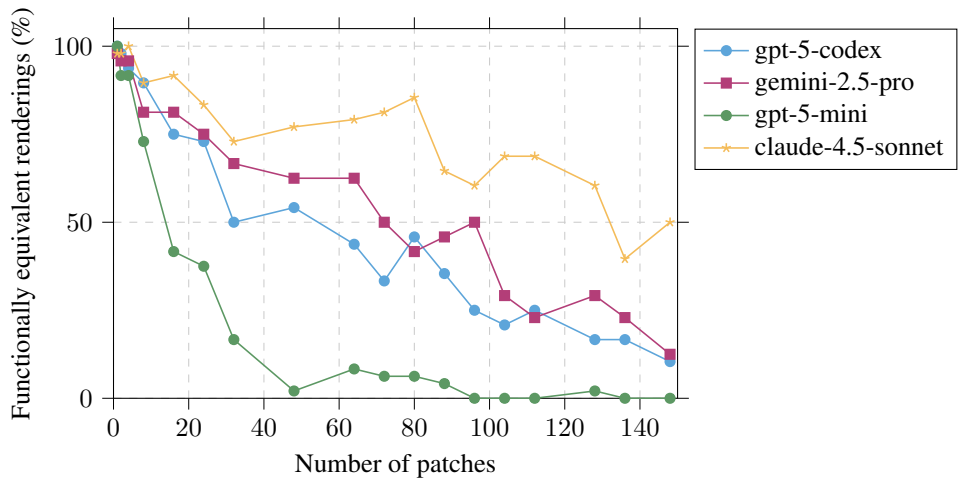

We do not argue that U0 or U1 representations should be used in practice in a long-context agentic setting; however, the specific failure mode exhibited here should be understood to apply wherever actions result in effects that are not sufficiently reflected in the observations and require ``mental bookkeeping''.

\section{The nuance of Chain of Thought}
\label{sec:nuance-cot}
For autoregressive language models, chain of thought-like mechanisms can serve as a form of ``memory'' to the model, which it can learn to utilize to externalize state tracking operations.
In recent literature, verifiable rewards have been used to optimize for better chains of thought, improving performance on downstream tasks.
For tasks where the tracked state is trivially represented in text, Transformers used with chain-of-thought prompting can solve input-length proportional problems in theory and practice.
However, input aggregation criticality remains a fundamental perceptual limitation of non-recurrence-complete models whenever the next token is generated from the current context.
Any prolonged chain of thought is therefore under constant pressure such that the number of sequential operations required to decode said context remains constant to avoid loss of information.
It is misleading to assume a truly unbounded state size even if the model can losslessly attend to all of the growing context.
An unbounded state size is only plausible if input aggregation can be performed in embarrassingly parallel fashion for the task at hand.
Additionally, a growing one-dimensional sequence of tokens as the model's state representation is potentially unsustainable because of eventual memory limitations.
The need to truncate and compact the sequence is effectively the same constraint recurrent neural networks face for cell capacity management,
where the cell has to ``learn to forget'' \citep{818041}, while remaining fully differentiable.

\section{Recurrence-Complete Frame-based Action Models}
Given the fundamental nature of the aggregation criticality problem, we suggest recurrence-completeness is a necessary property to reliably solve input-length proportional tasks.
However, existing recurrence-complete models are known to exhibit different issues, preventing them from being scaled to the degree that transformer-based models can.

Thus, we propose a scheme of recurrence-complete models that are partially parallelizable.
Instead of attempting full cross-temporal parallelization, we concede a natural sequential dependency of time,
however, the degree of sequentiality may not necessarily be one-to-one with the input-length.
This admits that some operations can be parallelized, while the overarching flow of time remains truly sequential.

We thus introduce a notion of ``frames'', which is distinctly different from traditional sequence modeling, which treats
the input-space as a one-dimensional sequence of tokens.
Instead, one frame is a fixed-length sequence that is asserted to be a complete representation of the input at time $t$.
Multiple frames form the sequence of observations $x_1, x_2, \dots, x_n$.

For example, a frame may be a 2D grid of pixels (e.g. an image) or a 2D grid of characters (e.g. a text terminal capture).

These frames are consumed by a ``frame-head'' that is tasked with embedding the frame into a latent space.

In our experiments, we employ a standard transformer with full attention and pooling operations to reduce the logical sequence length to learn tokenization.
Additionally, because the transformer is a time-parallel architecture, we reduce over the sequence of tokens with an LSTM to re-allocate all time-parallel compute to aid embedding formation.
Because there is no risk of leaking information within the frame-local sequence, as it is fully observable and not part of the prediction objective, we can safely employ pooling operations as opposed to relying on tokenization.

The frame-embeddings are then fed into the main sequence model, for which we employ a residual stack of LSTM cells interspersed with MLPs.

\subsection{The data}
To train models in unsupervised fashion, large amounts of labeled data are required. In practice, the only sufficiently large data source has been web text.
This has limited language models to consume a one dimensional sequence of tokens with the objective being to predict exactly the next token.

However, we point to a largely untapped source of labeled data, from which a sequence of actions can be generated - notably, git history.
Git version control history is a per-commit sequence of deltas, which can be applied in order to produce not only the current state of the code base,
but also all intermediate states.

From this information it is possible to reconstruct plausible text-editor keystrokes, which produce the current state of the code base.
We do this in a fully automated fashion using a custom shell and terminal muxer together with a tiny text editor, tools which optically replicate
common tooling such as bash and vim. By side-stepping expensive codepaths such as the xterm terminal emulator and manually populating
the character frame buffer appropriately, our C++ implementation can generate up to $200,000$ actions/s.
Additionally, we serialize this data with our own custom file format internally referred to as \texttt{termstreamxz} which is a lossless compression scheme
for terminal recordings. In practice we can achieve up to $300\times$ compression ratios with run-length, sliding-window reference and palette compression
by leveraging tight bit-packing.

This data should be thought of as a---though imperfect---substitute for actual terminal recordings from a hypothetical user typing out the repository to the extent
that the granularity of changes by commits reveals.

\begin{figure}[H]
  \centering
  \includegraphics[width=0.8\textwidth]{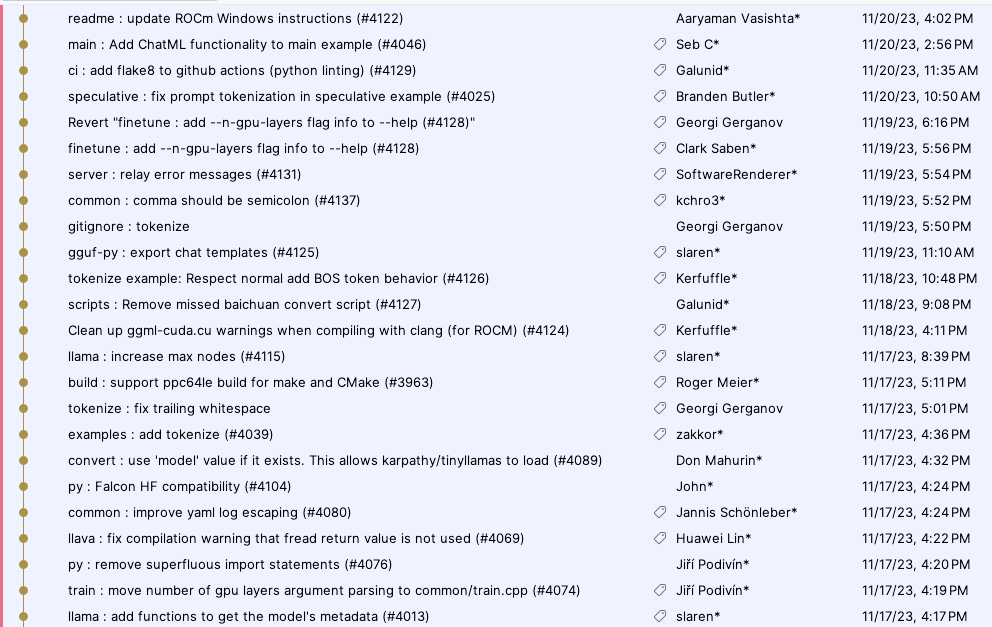}
  \caption{A git commit history with messages and author information displayed in graphical form.}
  \label{fig:git-history}
\end{figure}

\begin{figure}[H]
  \centering
  \includegraphics[width=0.8\textwidth]{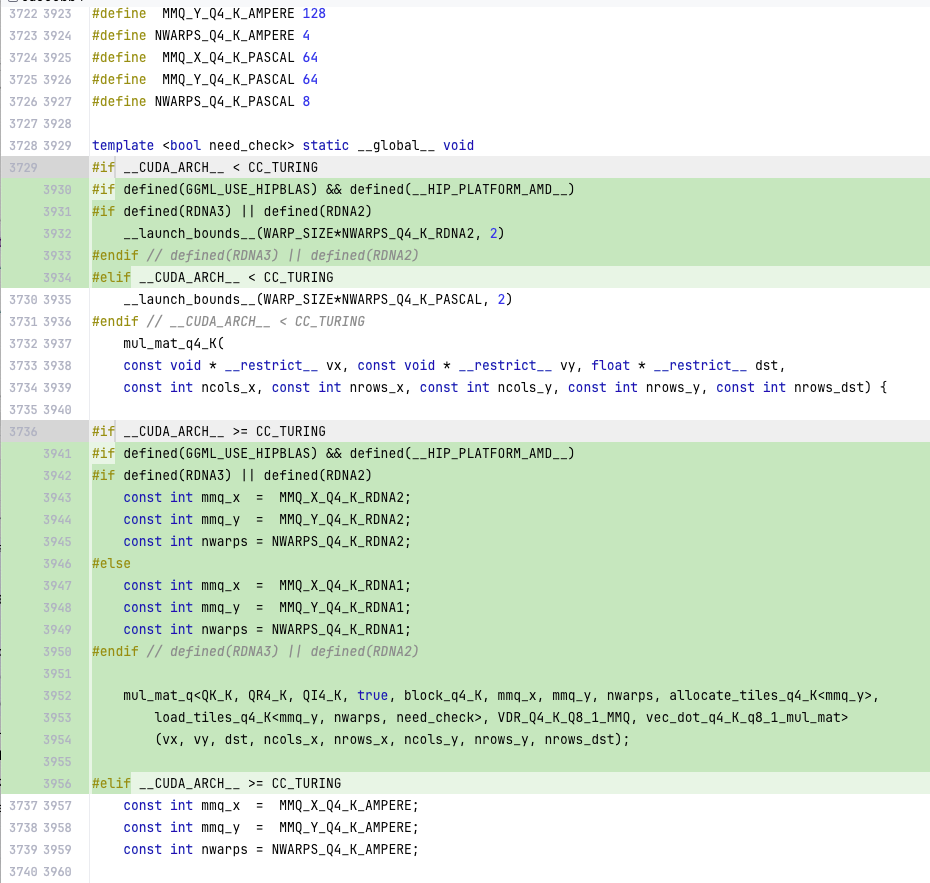}
  \caption{A unified git diff rendered in graphical form.}
  \label{fig:git-history-diff}
\end{figure}

\begin{figure}[H]
  \centering
  \includegraphics[width=0.8\textwidth]{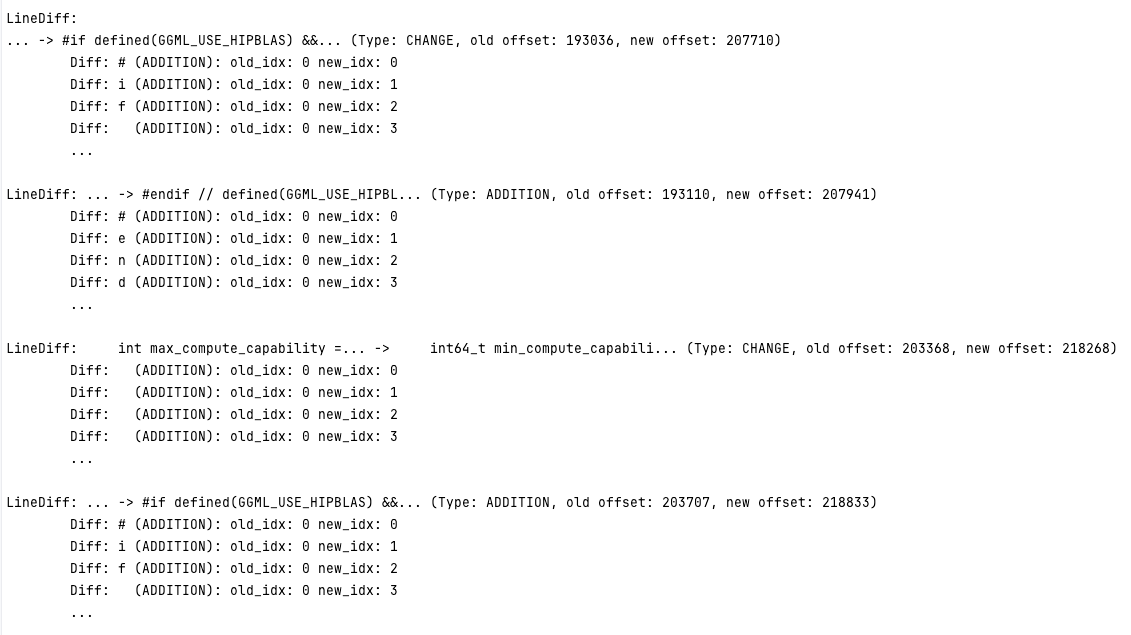}
  \caption{A visualization of the line and character diff format used. Character level diffs are strictly within the respective line boundaries to improve performance of the diffing algorithm \& also improve emitted actions "cosmetically".}
  \label{fig:line-and-char-diffs}
\end{figure}

The git history is iterated over in sequence using libgit2. Old and new file states are compared, and the equivalent editor actions are derived
according to the insert and cursor semantics of the text editor such that executing this sequence of actions will produce the new file state.

\begin{figure}[H]
  \centering
  \includegraphics[width=0.8\textwidth]{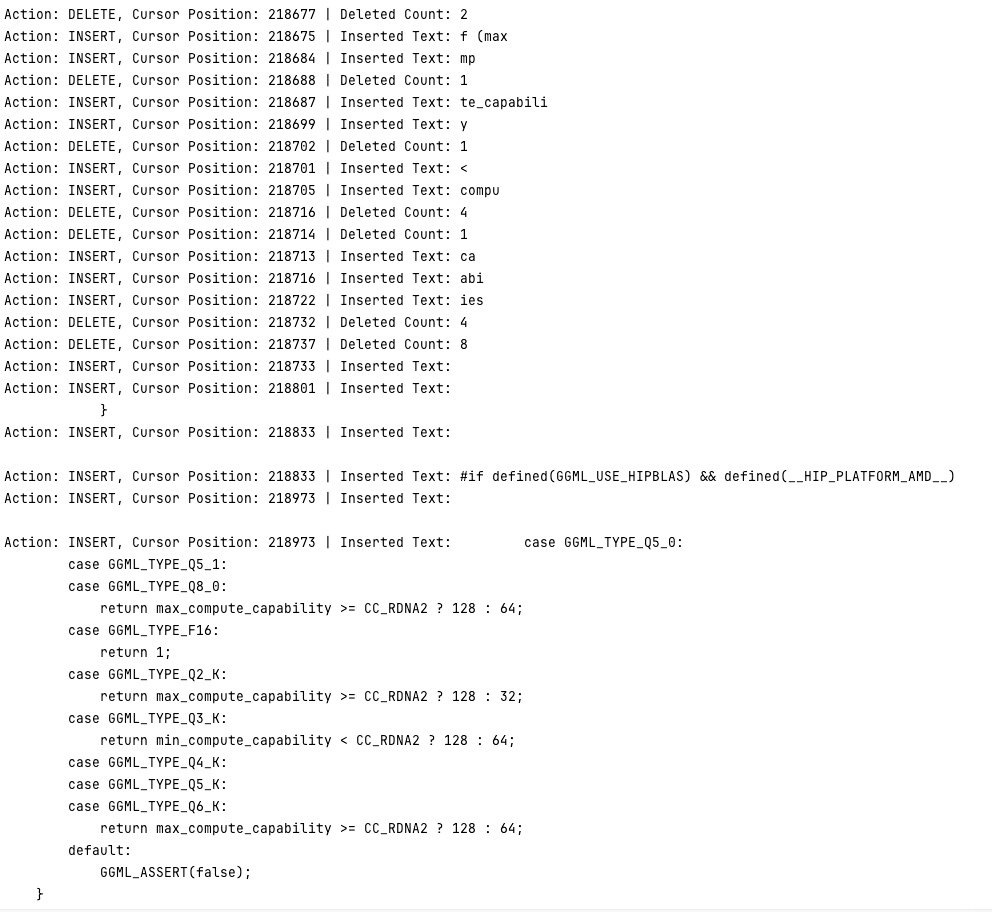}
  \caption{A sample of the editor actions. These are the final primitives executed by the text editor driver.}
  \label{fig:editor-actions}
\end{figure}

After applying the actions by driving the terminal emulator, we assert for safety reasons that the file state in the in-memory virtual file system matches the expected git end state -
a condition which even despite iterating over a sizable high-quality subset of GitHub repositories has never fired - even for long running repositories such as GCC, FFmpeg and LLVM.

\begin{figure}[H]
  \centering
  \includegraphics[width=0.8\textwidth]{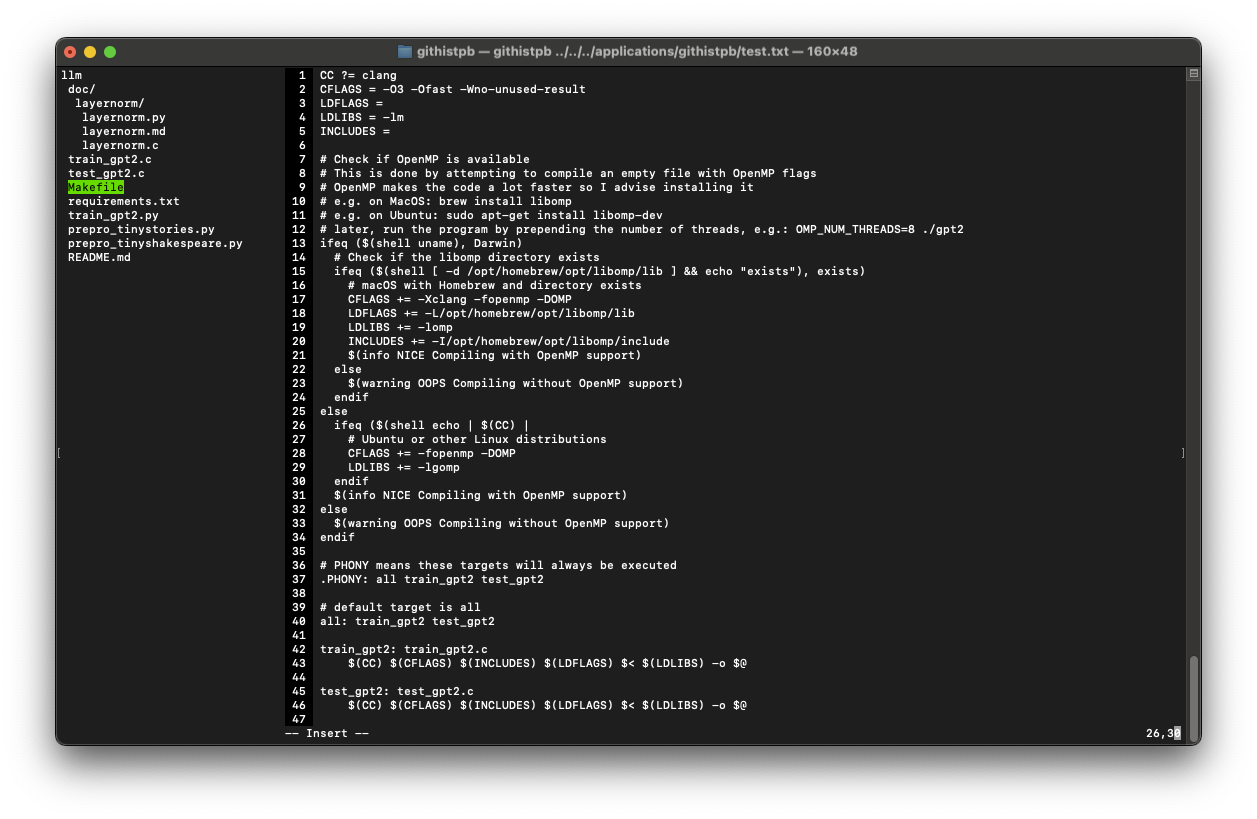}
  \caption{A sample of the terminal frame. The tiny-text-editor virtual process is opened inside a project folder. The left pane is a file browser, the right pane is a text editor with displayed line numbers.
  The "Makefile" is highlighted in the tree-view as the currently open file. The text editor is in "Insert" mode with the cursor at the end of line 26.
  }
  \label{fig:terminal-frame}
\end{figure}

After driving the terminal emulator, we capture the terminal frame buffer and encode it with our custom \texttt{termstreamxz} format.

The \texttt{termstreamxz} format is a video-esque format storing a sequence of frames, each of which is a sequence of characters along with color and styling options.
The format employs run-length encoding, sliding-window references and palette compression, leveraging tight bit-packing to achieve high compression ratios.
For example, an ``equivalence run'' is encoded as follows:
\bytefieldsetup{
  bitwidth=1.5em,
  bitheight=3.5em,
  boxformatting=\centering\scriptsize
}

\begin{figure}[H]
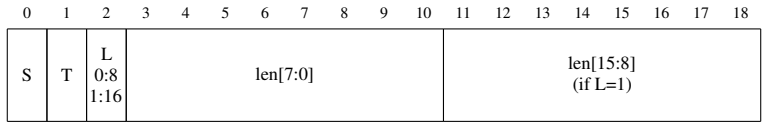

  \centering
  \begin{bytefield}{19}
    \bitheader{0-18}\\
    \bitbox{1}{S}
    \bitbox{1}{T}
    \bitbox{1}{L\\0:8\\1:16}
    \bitbox{8}{len[7:0]}
    \bitbox{8}{len[15:8]\\(if L=1)}
  \end{bytefield}
  \caption{Run encoding header structure}
\end{figure}

One bit is used to indicate that the following data is a ``special run'', which is either an equivalence run or a repeat run.
This bit is used to distinguish it from normal cell data, which runs may reference.
The next bit is used to indicate the type of special run, which is either an equivalence run (T=0) or a repeat run (T=1).
The next bit is used to indicate whether the run length is 8 or 16 bits long.
The next 8/16 bits are used to encode the run length.

With the above encoding we achieve on average a $300\times$ compression ratio compared to naive binary encoding of cell-states with 32 bits for codepoints and 8 bits for style channels.

Actions are stored separately where one action is either a single character or a null-terminated group of characters that form an xterm control sequence.
The actions are later tokenized with a greedy custom tokenizer of vocabulary size $20000$ that has been trained in case-aware fashion, taking advantage of the common means of communicating word boundaries in code, such as PascalCase, camelCase, snake\_case, etc.
Popularity of such a subword determines whether it should be considered a token, or constructed from smaller subwords.
The granularity of tokens determines the number of frames ``skipped'' between actions, given that a single action may imply more than one control-sequence to be emitted.

\begin{figure}[H]
  \centering
  \includegraphics[width=0.8\textwidth]{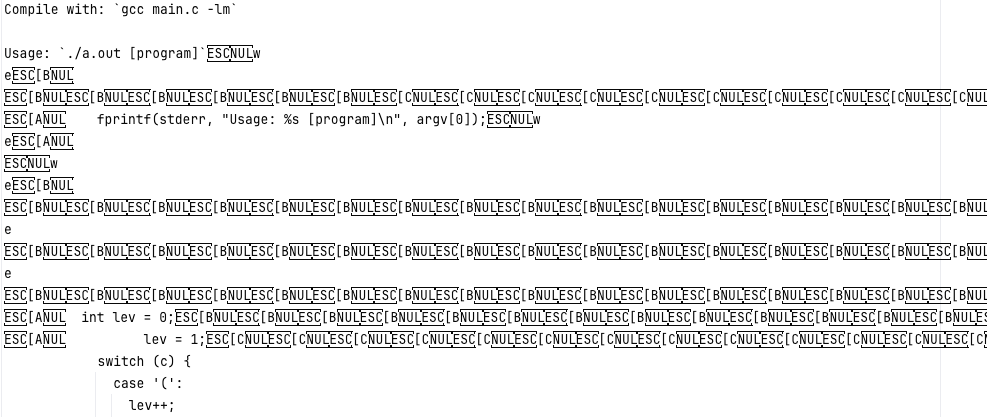}
  \caption{A sample of the action encoding format.}
  \label{fig:action-encoding}
\end{figure}

Logically, in inflated form, this data represents a sequence of ``text-video'' with next actions associated with each frame.

\begin{figure}[H]
  \centering
  \begin{tikzpicture}[x=1cm,y=1cm,font=\small]
    \def\W{6.0}
    \def\H{3.2}
    \def\dx{0.55}
    \def\dy{0.38}
    \def\Head{0.30}
    \def\gx{0.30}
    \def\gy{0.22}
    \def\Pad{0.03}
    \foreach \i/\lab/\dop in {
      -2/$x_{t-2}$/0.25,
      -1/$x_{t-1}$/0.45,
       0/$x_t$/1.00,
       1/$x_{t+1}$/0.75,
       2/$x_{t+2}$/0.55
    }{
      \begin{scope}[shift={(\i*\dx,\i*\dy)}]
        \path[fill=white, draw=black, draw opacity=\dop, line width=0.5pt, rounded corners=2pt]
             (0,0) rectangle (\W,\H);
        \path[fill=black!20, draw=black, draw opacity=\dop, rounded corners=2pt]
             (0,\H) rectangle (\W,\H-\Head);
        \begin{scope}
          \clip (0,0) rectangle (\W,\H-\Head-\Pad);
          \draw[xstep=\gx, ystep=\gy, very thin, draw=black!50, draw opacity=\dop]
               (0,0) grid (\W,\H-\Head-\Pad);
        \end{scope}
        \node[anchor=west, inner sep=2pt, text opacity=\dop] at (0.10,\H-\Head/2) {\lab};
      \end{scope}
    }
    \draw[->, thick] (-2.6*\dx,-2.6*\dy) -- (3.0*\dx,3.0*\dy) node[above right] {time};
  \end{tikzpicture}
  \caption{Stack of grid frames $x_t$ over time.}
  \label{fig:frame-stack}
\end{figure}

These frames are then fed into our model, which consists of a frame-head and a main sequence model consisting of a residual stack of LSTM cells interspersed with MLPs to predict the next action at the given time step.

\begin{figure}[H]
  \centering
  \resizebox{\linewidth}{!}{
  \begin{tikzpicture}[
      font=\small,
      >={Latex[length=2mm]},
      every label/.style={font=\scriptsize, inner sep=1pt},
      frame/.style={draw, rounded corners=2pt, inner sep=2pt,
                    minimum width=16mm, minimum height=12mm, fill=gray!5},
      gridline/.style={draw=black!25, line width=0.3pt},
      wrap/.style={draw, rounded corners=4pt, inner sep=4pt,
                   minimum height=12mm, minimum width=92mm, fill=orange!8},
      unit/.style={draw, rounded corners=2pt, align=center, inner sep=1.5pt,
                   minimum height=6mm, minimum width=13mm, fill=orange!15},
      mlpunit/.style={draw, rounded corners=2pt, align=center, inner sep=1.5pt,
                      minimum height=6mm, minimum width=13mm, fill=blue!15},
      act/.style={draw, circle, inner sep=0pt, minimum size=8mm, fill=green!12, font=\scriptsize},
      arr/.style={->, line width=1pt},
      bus/.style={line width=1pt, draw=black},
      link/.style={line width=1pt},
      conn/.style={draw=gray, densely dashed, line width=1pt},
      junction/.style={
        draw, circle, fill=white,
        inner sep=0pt, minimum size=2.6mm,
        line width=0.5pt, font=\scriptsize
      }
    ]
    \pgfdeclarelayer{behindcells}
    \pgfsetlayers{background,behindcells,main}
    \def\HeadH{12mm}
    \def\HeadHalf{6mm}
    \def\NeckH{6mm}
    \def\NeckHalf{3mm}
    \def\TrapL{6mm}
    \def\RectL{10mm}
    \def\TriL{4mm}
    \def\HeadGap{1mm}
    \tikzset{headshape/.style={draw, fill=blue!8, line join=round}}
    \def\rowgap{1.8}
    \def\wrapgap{14mm}
    \def\leftgap{2mm}
    \def\rightgap{6mm}
    \def\vlineoffA{3mm}
    \def\vlineoffB{3mm}
    \def\pairgapA{4mm}
    \def\pairgapB{4mm}
    \foreach \i/\name in {0/{x_{t-1}},1/{x_t},2/{x_{t+1}}}{
      \begin{scope}[shift={(0,-\i*\rowgap)}]
        \node[frame, label=left:{\(\name\)}] (frame\i) {};
        \path (frame\i.south west) coordinate (sw\i);
        \draw[gridline] (sw\i) ++(0,4mm) -- ++(16mm,0);
        \draw[gridline] (sw\i) ++(0,8mm) -- ++(16mm,0);
        \draw[gridline] (sw\i) ++(5.333mm,0) -- ++(0,12mm);
        \draw[gridline] (sw\i) ++(10.666mm,0) -- ++(0,12mm);
      \end{scope}
    }
    \foreach \i in {0,1,2}{
      \coordinate (headIn\i) at ($(frame\i.east)+(22mm,0)$);
      \path[headshape]
        ($(headIn\i)+(0,\HeadHalf)$) -- ($(headIn\i)+(0,-\HeadHalf)$) --
        ($(headIn\i)+(\TrapL,-\NeckHalf)$) --
        ($(headIn\i)+(\TrapL,\NeckHalf)$) -- cycle;
      \coordinate (rectIn\i) at ($(headIn\i)+(\TrapL+\HeadGap,0)$);
      \path[headshape]
        ($(rectIn\i)+(0,\NeckHalf)$) rectangle
        ($(rectIn\i)+(\RectL,-\NeckHalf)$);
      \coordinate (triIn\i) at ($(rectIn\i)+(\RectL+\HeadGap,0)$);
      \path[headshape]
        ($(triIn\i)+(0,\NeckHalf)$) --
        ($(triIn\i)+(0,-\NeckHalf)$) --
        ($(triIn\i)+(\TriL,0)$) -- cycle;
      \coordinate (headOut\i) at ($(triIn\i)+(\TriL,0)$);
      \draw[decorate,decoration={brace,amplitude=5pt,mirror}]
        ($(headIn\i)+(0,-\HeadHalf-2pt)$) -- ($(headOut\i)+(0,-\HeadHalf-2pt)$)
        node[midway,yshift=-8pt,font=\scriptsize]{Frame head};
      \draw[arr] (frame\i.east) -- (headIn\i);
    }
    \foreach \i in {0,1,2}{
      \begin{pgfonlayer}{background}
        \node[wrap, anchor=west] (lstm\i) at ($(headOut\i)+(\wrapgap,0)$) {};
      \end{pgfonlayer}
      \draw[bus] (headOut\i) -- (lstm\i.west);
      \coordinate (cellApos\i) at ($(lstm\i.north west)+(\leftgap,0)+(2mm,-2mm)$);
      \node[unit, anchor=north west] (cellA\i) at (cellApos\i) {LSTM Cell};
      \node[mlpunit, anchor=north west] (mlpA\i) at ($(cellA\i.north east)+(\pairgapA,0)$) {MLP};
      \node[unit, anchor=north west] (cellB\i) at ($(mlpA\i.north east)+(\pairgapB,0)$) {LSTM Cell};
      \node[mlpunit, anchor=north west] (mlpB\i) at ($(cellB\i.north east)+(\pairgapA,0)$) {MLP};
      \draw[link] ($(cellA\i.east)!0.5!(mlpA\i.west)$) -- (mlpA\i.west);
      \draw[link] ($(mlpA\i.east)!0.5!(cellB\i.west)$) -- (cellB\i.west);
      \draw[link] ($(cellB\i.east)!0.5!(mlpB\i.west)$) -- (mlpB\i.west);
      \coordinate (rgapL\i) at ($(lstm\i.east)+(-\rightgap,0)$);
      \node[anchor=east] (dots\i) at (rgapL\i) {$\cdots$};
      \coordinate (baseY\i) at ($(lstm\i.south west)+(0,2mm)$);
      \coordinate (busA\i)  at ($(cellA\i.south |- baseY\i)$);
      \coordinate (feedB\i) at ($(cellB\i.south |- baseY\i)$);
      \coordinate (gapC\i)  at ($(rgapL\i)!0.5!(lstm\i.east)$);
      \coordinate (joinB\i) at ($(gapC\i |- baseY\i)$);
      \coordinate (joinM\i) at ($(gapC\i |- lstm\i)$);
      \coordinate (outL\i)  at ($(lstm\i.east)+(0.8mm,0)$);
      \coordinate (entry\i) at (lstm\i.west);
      \coordinate (inner\i) at ($(entry\i)+(\leftgap,0)$);
      \coordinate (turn\i)  at ($(inner\i |- cellA\i.west)$);
      \draw[bus] (busA\i) -| (turn\i);
      \draw[bus] (busA\i) -- (joinB\i);
      \draw[bus, shorten >=0.5\pgflinewidth] (joinB\i) -- (joinM\i) -- (lstm\i.east);
      \coordinate (busMA\i) at ($(mlpA\i.south |- baseY\i)$);
      \coordinate (busMB\i) at ($(mlpB\i.south |- baseY\i)$);
      \draw (cellA\i.south) -- (busA\i);
      \draw (cellB\i.south) -- (feedB\i) -- (joinB\i);
      \draw (mlpA\i.south) -- (busMA\i);
      \draw (mlpB\i.south) -- (busMB\i);
      \node[junction] at (busA\i) {$+$};
      \node[junction] at (feedB\i) {$+$};
      \node[junction] at (busMA\i) {$+$};
      \node[junction] at (busMB\i) {$+$};
      \coordinate (midA\i)      at ($(cellA\i.east)!0.5!(mlpA\i.west)$);
      \coordinate (midA_bus\i)  at ($(midA\i |- baseY\i)$);
      \coordinate (midAB\i)     at ($(mlpA\i.east)!0.5!(cellB\i.west)$);
      \coordinate (midAB_bus\i) at ($(midAB\i |- baseY\i)$);
      \coordinate (midB\i)      at ($(cellB\i.east)!0.5!(mlpB\i.west)$);
      \coordinate (midB_bus\i)  at ($(midB\i |- baseY\i)$);
      \draw[bus,shorten <=-0.5\pgflinewidth] (midA\i)  -- (midA_bus\i);
      \draw[bus,shorten <=-0.5\pgflinewidth] (midAB\i) -- (midAB_bus\i);
      \draw[bus,shorten <=-0.5\pgflinewidth] (midB\i)  -- (midB_bus\i);
      \draw[link] (entry\i) -- (inner\i) -- (turn\i) -- (cellA\i.west);
    }
    \foreach \i/\lab in {0/$a_{t-1}$,1/$a_t$,2/$a_{t+1}$}{
      \node[act] (act\i) at ($(lstm\i.east)+(12mm,0)$) {\lab};
      \draw[arr] (lstm\i.east) -- (outL\i) -- (act\i.west);
    }
    \begin{pgfonlayer}{behindcells}
      \foreach \i in {0,1,2}{
        \coordinate (colA\i) at ($(cellA\i.south)+(\vlineoffA,0)$);
        \coordinate (colB\i) at ($(cellB\i.south)+(\vlineoffB,0)$);
      }
      \draw[conn] ($(colA0|-cellA0.center)$) -- ($(colA1|-cellA1.center)$);
      \draw[conn] ($(colA1|-cellA1.center)$) -- ($(colA2|-cellA2.center)$);
      \node[font=\scriptsize, align=center, inner sep=2pt, fill=white, text=black!60]
        at ($(colA1|-cellA1.center)+(0,8.5mm)$) {h + CEC};
      \draw[conn] ($(colB0|-cellB0.center)$) -- ($(colB1|-cellB1.center)$);
      \draw[conn] ($(colB1|-cellB1.center)$) -- ($(colB2|-cellB2.center)$);
      \node[font=\scriptsize, align=center, inner sep=2pt, fill=white, text=black!60]
        at ($(colB1|-cellB1.center)+(0,8.5mm)$) {h + CEC};
    \end{pgfonlayer}
    \node[black!60] at ($(lstm0.north)+(0,6mm)$) {$\vdots$};
    \node[black!60] at ($(lstm2.south)+(0,-6mm)$) {$\vdots$};
  \end{tikzpicture}
  }
  \caption{Frame-based Action Model with Frame-Head and Main Sequence model.}
  \label{fig:frame-based-action-model}
\end{figure}

Additionally, our data contains ``dummy vcs'' actions which mirror the real commit messages for the code in question, thus providing implicit conditioning for subsequent actions to be described by the commit message.
We find that commit messages are often extremely high-quality descriptions of the contribution for repositories such as LLVM.

\begin{figure}[H]
  \centering
  \includegraphics[width=0.8\textwidth]{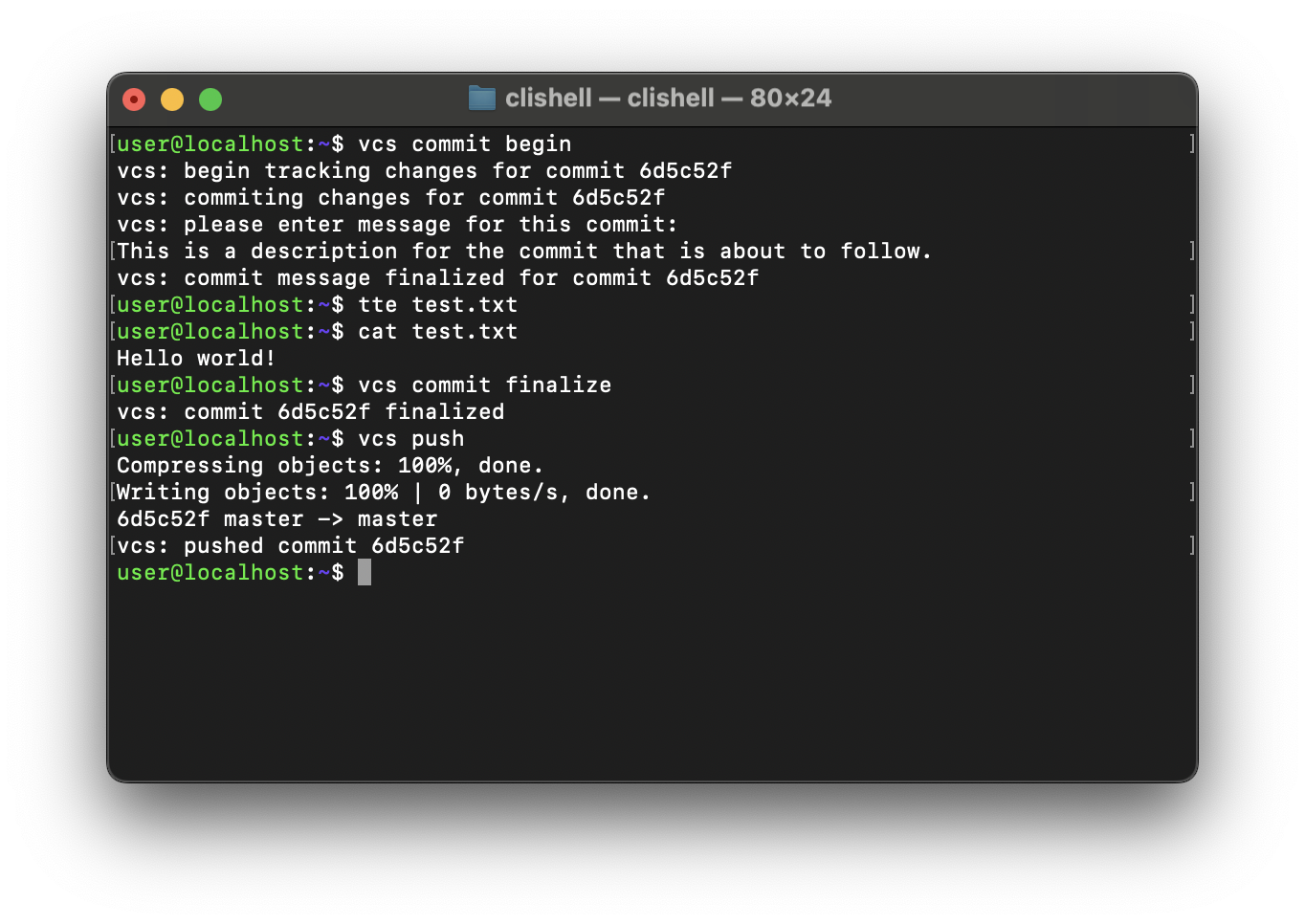}
  \caption{Example VCS interaction}
  \label{fig:example-vcs-interaction}
\end{figure}

\begin{figure}[H]
  \centering
  \includegraphics[width=0.8\textwidth]{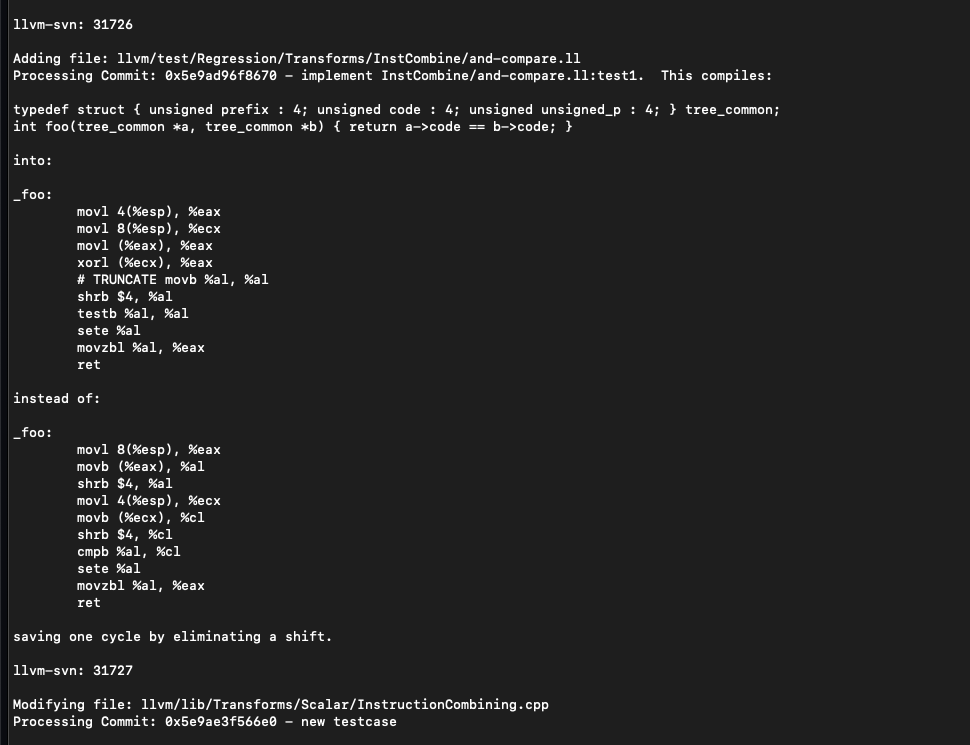}
  \caption{Example commit message}
  \label{fig:example-commit-msg}
\end{figure}

\subsection{Frame-Head}
The frame head is a transformer-based model that is tasked with embedding the frame into a latent space.
It does so by applying full self-attention to the input frame sequence interspersed with pooling operations to reduce sequence length, accomplishing the same goal as tokenization.
We note that while actions are tokenized, input cell-state to the frame-head remains character level.
Finally, the output is fed into a terminating LSTM cell to produce the frame embedding.

\begin{figure}[H]
  \centering
  \begin{tikzpicture}[
    font=\small,
    node distance=10mm,
    >={Latex},
    tensor/.style={draw,thick,rounded corners=2pt,minimum height=9mm,minimum width=42mm,align=center,fill=blue!6},
    op/.style={draw,thick,rounded corners=2pt,minimum height=10mm,minimum width=42mm,align=center,fill=orange!12},
    smallop/.style={draw,thick,rounded corners=2pt,minimum height=9mm,minimum width=42mm,align=center,fill=yellow!18},
    mlp/.style={draw,thick,rounded corners=2pt,minimum height=9mm,minimum width=42mm,align=center,fill=blue!6},
    blockwrap/.style={draw,rounded corners=3pt,thick,inner xsep=7pt,inner ysep=10pt},
    arrow/.style={-Latex,thick},
    brace/.style={decorate,decoration={brace,amplitude=4pt}},
  ]
  \def\colA{7mm}
  \def\colB{12mm}
  \def\labelsep{2.5mm}
  \def\colL{4mm}
  \node[tensor] (inp) {Input tokens\\$\mathbf{X}\in\mathbb{R}^{B\times T\times C}$\\(with $C=n_{\text{embed}}$)};
  \node[op, below=of inp] (attn) {Full Self-Attention\\(multi-head)\\$(B,T,C)\to(B,T,C)$};
  \node[mlp, below=of attn] (mlp0) {MLP $(C\rightarrow 4C \rightarrow C)$};
  \node[smallop, below=of mlp0] (avg1) {AvgPool1D over $T$\\(stride $s_1$)\\$(B,T,C)\to(B,T',C)$};
  \node[align=center, below=of avg1] (dots) {$\hdots$};
  \node[op, below=of dots] (attn2) {Full Self-Attention\\(multi-head)\\$(B,T',C)\to(B,T',C)$};
  \node[mlp, below=of attn2] (mlp2) {MLP $(C\rightarrow 4C \rightarrow C)$};
  \node[align=center, below=of mlp2] (dots2) {$\hdots$};
  \node[op, below=of dots2] (lstm) {Terminating LSTM\\(hidden size $=n_{\text{embed}}$)\\use last $h$\\$(B,T',C)\to(B,C)$};
  \node[tensor, below=of lstm] (out) {Aggregated token\\$\mathbf{y}\in\mathbb{R}^{B\times C}$\\($C=n_{\text{embed}}$)};
  \draw[arrow] (inp) -- (attn);
  \draw[arrow] (attn) -- (mlp0);
  \draw[arrow] (mlp0) -- (avg1);
  \draw[arrow] (avg1) -- (dots) -- (attn2);
  \draw[arrow] (attn2) -- (mlp2) -- (dots2) -- (lstm);
  \draw[arrow] (lstm) -- (out);
  \begin{pgfonlayer}{background}
    \node[blockwrap, fit=(attn)(mlp0)] (blk0) {};
    \node[blockwrap, fit=(attn2)(mlp2)] (blk1) {};
  \end{pgfonlayer}
  \coordinate (nxX)  at ($ (blk0.west) + (-\colL,0) $);
  \coordinate (nxTop) at (nxX |- blk0.north);
  \coordinate (nxBot) at (nxX |- avg1.south);
  \draw[decorate,decoration={brace,mirror,amplitude=4pt}]
    (nxTop) -- (nxBot)
    node[midway, anchor=east, xshift=-\labelsep] (labNx) {$N\times$};
  \coordinate (nx2X)  at ($ (blk1.west) + (-\colL,0) $);
  \coordinate (nx2Top) at (nx2X |- blk1.north);
  \coordinate (nx2Bot) at (nx2X |- blk1.south);
  \draw[decorate,decoration={brace,mirror,amplitude=4pt}]
    (nx2Top) -- (nx2Bot)
    node[midway, anchor=east, xshift=-\labelsep] (labNx2) {$N\times$};
  \draw[brace]
    ($(blk0.north east)+(\colA, 1.2mm)+(2.5mm, 0mm)$) -- ($(blk0.south east)+(\colA,-1.2mm)+(2.5mm, 0mm)$)
    node[midway, anchor=west, xshift=\labelsep] (labBlk0) {$T$ time steps};
  \draw[brace]
    ($(avg1.north east)+(\colB, 1.2mm)$) -- ($(avg1.south east)+(\colB,-1.2mm)$)
    node[midway, anchor=west, xshift=\labelsep] (labAvg1) {$T\to T'$ (downsample)};
  \draw[brace]
    ($(blk1.north east)+(\colA, 1.2mm)$) -- ($(blk1.south east)+(\colA,-1.2mm)$)
    node[midway, anchor=west, xshift=\labelsep] (labBlk1) {$T'$ time steps};
  \draw[brace]
    ($(lstm.north east)+(\colA, 1.2mm)$) -- ($(lstm.south east)+(\colA,-1.2mm)$)
    node[midway, anchor=west, xshift=\labelsep] (labLstm) {$T'\to 1$ (take last state)};
  \begin{pgfonlayer}{background}
    \node[draw,dashed,rounded corners=3pt,thick,
          inner xsep=14pt, inner ysep=16pt,
          fit=(attn)(mlp0)
              (avg1)(dots)
              (attn2)(mlp2)(dots2)
              (lstm)
              (labNx)(labNx2)
              (labBlk0)(labAvg1)(labBlk1)(labLstm)] (framebox) {};
  \end{pgfonlayer}
  \end{tikzpicture}
  \caption{The frame head architecture.}
  \label{fig:frame-head}
\end{figure}
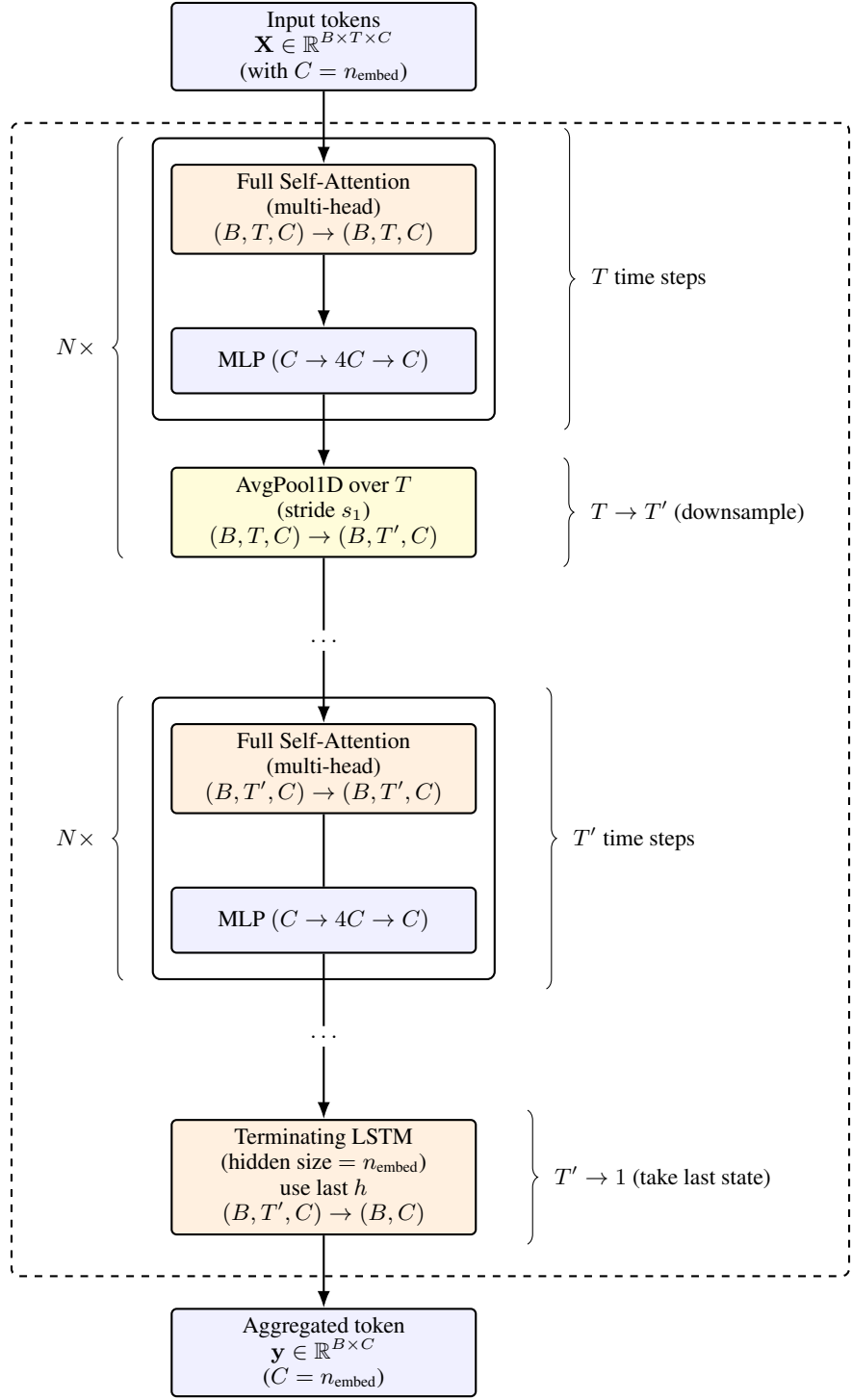

\subsection{Streaming Backpropagation and Recomputation}
To backpropagate through time to what amounts to potentially thousands of frame head forward passes, we employ full recomputation of frame head activations for the backward pass.
Never more than one frame head is backpropagated through concurrently to minimize memory usage.
Additionally, the LSTMs of the main sequence model page activations saved for backpropagation to host memory to be streamed back to the GPU as the backward pass progresses.
This can be achieved in chunked pre-fetching fashion such that transfer occurs concurrently with the last frame-head's gradient contribution computation.
This side-steps the vast amount of GPU memory that would otherwise be required to backpropagate through this model.
This keeps GPU memory usage roughly at $O(1)$ as a function of sequence length at the cost of constant factor increase in wall-time.

\subsection{Experiments}
\label{sec:fbam-experiments}
In our experiments, we employ the Muon optimizer \citep{jordan2024muon} with a fixed learning rate of $3\times 10^{-3}$ for all matrix-parameters while using AdamW \citep{loshchilov2019decoupled} for all other parameters with a fixed learning rate of $3\times 10^{-4}$, betas of $(0.9, 0.95)$ and weight decay of $0.01$
along with a total batch size of $512$.
We note that ``sequence length'' here refers to the number of frames used for a single training example.
Full backpropagation through time is employed across the entire sequence of frames in streaming fashion.
Each frame consists of $48 \times 160 = 7680$ input cells of ``per-frame'' sequence length.

\subsubsection{GitHub Compilers and Interpreters Dataset}
\label{sec:compilers-and-interpreters-experiments}
The following set of experiments was conducted on the ``compilers and interpreters'' dataset, a subset of GitHub filtered for projects that implement toy compilers and interpreters.
Sustained accuracy is measured as the average number of correct frames in a row without interruption as per top-1 sampling on a held-out validation set.
Loss is measured as training loss at 1000 steps.

\begin{table}[H]
  \centering
  \footnotesize
  \setlength{\tabcolsep}{3pt}
  \begin{tabular}{ccccccccccc}
    \toprule
    $n_{params}$ & $n_{hidden}$ & $n_{layer}$ & $n_{pool}$ & $d_{model}$ & $n_{seq}$ & $n_{frames}$ & LR & BS & loss$_{train}$ & acc$_{sust}$ \\
    \midrule
    103M & 768 & 6 & 2 & 768 & 2 & 2 & $3\times 10^{-3}$ & 512 & 2.33 & 1.06 \\
    103M & 768 & 6 & 2 & 768 & 2 & 16 & $3\times 10^{-3}$ & 512 & 1.11 & 3.02 \\
    103M & 768 & 6 & 2 & 768 & 2 & 128 & $3\times 10^{-3}$ & 512 & 0.25 & 141 \\
    145M & 768 & 12 & 2 & 768 & 2 & 128 & $3\times 10^{-3}$ & 512 & 0.23 & 150 \\
    82M & 768 & 3 & 2 & 768 & 2 & 256 & $3\times 10^{-3}$ & 512 & 0.17 & 212 \\
    \bottomrule
  \end{tabular}
\end{table}

\subsubsection{GitHub Technical Excellence Dataset}
\label{sec:github-technical-excellence-experiments}

The following set of experiments was conducted on a dataset generated from a high-quality subset of GitHub repositories filtered for ``technical excellence'', resulting in 1.6TB of highly compressed training data.
Loss is measured as training loss at 4000 steps.

\begin{table}[H]
  \centering
  \footnotesize
  \setlength{\tabcolsep}{3pt}
  \begin{tabular}{ccccccccccc}
    \toprule
    $n_{params}$ & $n_{hidden}$ & $n_{layer}$ & $n_{pool}$ & $d_{model}$ & $n_{seq}$ & $n_{frames}$ & LR & BS & loss$_{train}$ & acc$_{sust}$ \\
    \midrule
    82M & 768 & 3 & 2 & 768 & 2 & 2 & $3\times 10^{-3}$ & 512 & 4.69 & 0.35 \\
    82M & 768 & 3 & 2 & 768 & 2 & 2 & $3\times 10^{-3}$ & 1024 & 4.22 & 0.28 \\

    82M & 768 & 3 & 2 & 768 & 2 & 4 & $3\times 10^{-3}$ & 512 & 3.01 & 0.45 \\
    82M & 768 & 3 & 2 & 768 & 2 & 4 & $3\times 10^{-3}$ & 1024 & 2.98 & 0.63 \\

    82M & 768 & 3 & 2 & 768 & 2 & 16 & $3\times 10^{-3}$ & 512 & 1.88 & 2.03 \\
    82M & 768 & 3 & 2 & 768 & 2 & 16 & $3\times 10^{-3}$ & 1024 & 1.61 & 1.99 \\

    82M & 768 & 3 & 2 & 768 & 2 & 128 & $3\times 10^{-3}$ & 512 & 1.01 & 3.71 \\
    82M & 768 & 3 & 2 & 768 & 2 & 128 & $3\times 10^{-3}$ & 1024 & 0.82 & 4.32 \\
    737M & 4096 & 3 & 2 & 1024 & 2 & 128 & $3\times 10^{-3}$ & 512 & 0.73 & 4.56 \\

    82M & 768 & 3 & 2 & 768 & 2 & 512 & $3\times 10^{-3}$ & 512 & 0.71 & 5.24 \\
    82M & 768 & 3 & 2 & 768 & 2 & 1024 & $3\times 10^{-3}$ & 512 & 0.61 & 6.20 \\
    \bottomrule
  \end{tabular}
\end{table}

\begin{figure}[H]
  \centering
  \includegraphics[width=0.8\textwidth]{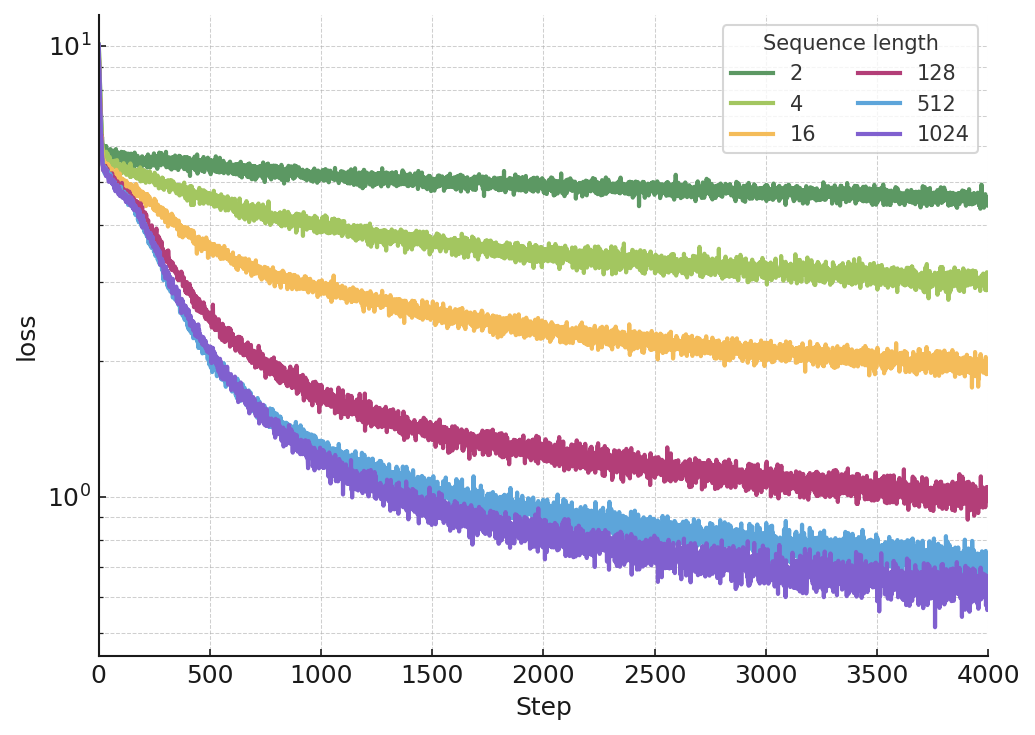}
  \caption{Loss for models with different sequence lengths as a function of step count.}
  \label{fig:technical-excellence-loss}
\end{figure}

\subsection{Scaling Trends}
Training these models, we observe a clear trend of faster convergence as the trained sequence length increases.
Scaling sequence length in frame count linearly increases the runtime per step; however, this extra cost is amortized as training progresses, and beyond a certain point the longer-sequence runs overtake shorter ones in loss as a function of wall time.
As we continue to explore these scaling laws our primary focus remains on pushing sequence length while increasing cell capacity when needed.

Specifically, training loss at a fixed step \(s\) follows a strict power law in the sequence length \(L\):
\begin{equation}
\mathrm{loss}(L\mid s) \approx A(s)\, L^{-\alpha(s)} .
\end{equation}
At \(s=400\) we obtain \(\alpha(400)=0.129\) and \(A(400)\approx 5.65\) with \(R^2=0.971\); by \(s=650\) the exponent increases to \(\alpha(650)=0.196\) with \(A(650)\approx 5.80\) (\(R^2=0.989\)).
A convenient rule of thumb is the per-doubling improvement \(2^{-\alpha(s)}\): at 400 steps a doubling of \(L\) reduces loss by \(\approx 8.6\%\), and by 650 steps by \(\approx 12.7\%\).
By \(s=4000\) we measure \(\alpha(4000)\approx 0.318\) and \(A(4000)\approx 4.96\) (\(R^2\approx 0.993\)), consistent with approaching a plateau.

\subsubsection{Evolution of the Scaling Exponent}
The exponent \(\alpha(s)\) grows during early training and then plateaus. A simple saturating exponential captures this dynamic:
\begin{equation}
\alpha(s) = \alpha_{\infty}\!\left(1 - e^{-s/\tau}\right),
\qquad
\alpha_{\infty} \approx 0.308,\ \tau \approx 720\ \text{steps}.
\end{equation}
Thus, the loss ratio between two lengths satisfies
\begin{equation}
\frac{\mathrm{loss}(L_2\mid s)}{\mathrm{loss}(L_1\mid s)} \approx \left(\frac{L_2}{L_1}\right)^{-\alpha(s)} ,
\end{equation}
so at the plateau \(\alpha_{\infty}\) a doubling of \(L\) yields a sustained \(\approx 19\text{--}20\%\) reduction in loss \(\bigl(2^{-\alpha_{\infty}}\approx 0.808\bigr)\).

\paragraph{Implications for wall-time.}
If wall-time per step scales linearly with sequence length, then equal-time comparisons follow
\begin{equation}
\mathrm{loss}(t, L) \;\approx\; A\!\left(\tfrac{\gamma t}{L}\right)\, L^{-\alpha\!\left(\tfrac{\gamma t}{L}\right)} ,
\end{equation}
where $\gamma$ converts wall-time to steps. Under mild conditions where $A(s)$ varies slowly or saturates, longer sequences amortize their extra per-step cost and eventually dominate on a loss-versus-time plot.

\paragraph{Assumptions and claim.}
We assume
\begin{equation}
\mathrm{loss}(L\mid s) \;=\; A(s)\,L^{-\alpha(s)}\,\bigl(1+r(L,s)\bigr),
\label{eq:basic-scaling}
\end{equation}
with $A(s)\to A_\infty>0$, $\alpha(s)\to \alpha_\infty>0$ as $s\to\infty$, and a uniform multiplicative error satisfying $\sup_{L\ge 1}|r(L,s)|\to 0$ as $s\to\infty$. Wall-time per step scales linearly so $s(L,t)=\gamma t/L$.

\textbf{Claim.} For any $L_2>L_1$ there exists $T$ such that for all $t\ge T$ one has $\mathrm{loss}(t,L_2)<\mathrm{loss}(t,L_1)$. The proof appears in Appendix \ref{sec:walltime-amortization-proof}.

This effect can be seen to manifest empirically in Figure \ref{fig:walltime-amortization}:
\begin{figure}[H]
  \centering
  \includegraphics[width=0.8\textwidth]{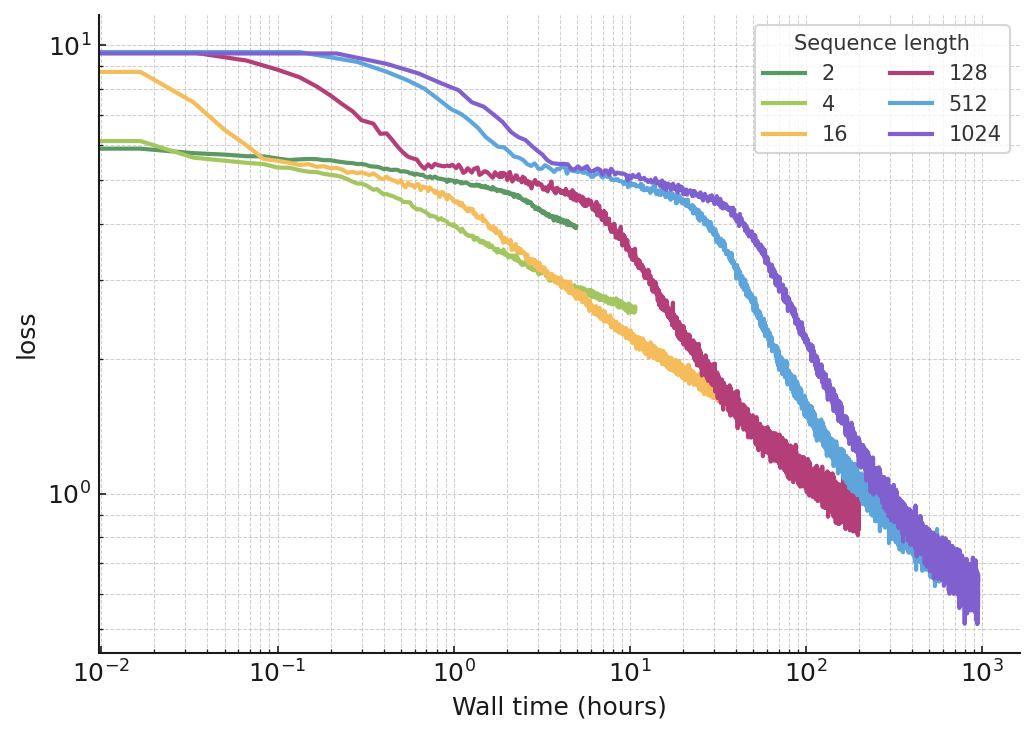}
  \caption{Wall-time amortization. Loss as a function of wall time "catches up" to the shorter sequence length runs. Trend is made more apparent by using log-scale.}
  \label{fig:walltime-amortization}
\end{figure}

This means that linearly increasing wall time per step is amortized through faster convergence, resulting in ultimately lower loss as a function of wall time
as the wall time approaches infinity.
In such a setting it never makes sense to deliberately train a shorter sequence length model for reasons other than practicality.

We phrase the scaling law as a function of wall time explicitly to highlight that while batch size must be chosen appropriately, it can be done so without increasing wall time.
Therefore the important relation to capture is training loss achievable at a given sequence length in relation to wall time with batch size being independent of it.

\subsection{Isn't this just more tokens?}
Because for the training-runs in question the batch size was kept constant, the amount of actions trained on per update increases.
We thus analyze the effect of decreasing batch size to keep the amount of actions per update constant as we scale sequence length, keeping $batch\_size \times sequence\_length$ constant.
We find that while the decrease in optimization signal does slow down convergence as a function of step count, lower sequence length runs are still consistently outperformed by their higher sequence length counterparts \footnote{See Appendix \ref{sec:fixating-number-of-actions-per-update} for details on fixing the number of actions per update}.
Additionally, increasing the batch size at any given sequence length only improves convergence speed up to a certain point.
Keeping the number of actions per update constant, higher sequence length runs consistently outperform higher batch size runs.
Therefore we suggest that there exists a critical batch size after which sequence length must be scaled to obtain lower loss\footnote{See Appendix \ref{sec:critical-batch-size} for details on the critical batch size}.

While higher sequence length runs perform more FLOPs than higher batch size runs at equal action counts, their frame head compute budget remains identical.
We note however that the number of flops performed by the LSTM-based main sequence model itself remains a rather small fraction of the overall performed FLOPs (approx. 6\% of the total forward FLOPs incl. MLPs)\footnote{See Appendix \ref{sec:model-flops} for details on the LSTM FLOPs.}.
Under these assumptions, keeping the number of actions per update constant also roughly preserves the FLOP budget.

\subsection{What about parameter scaling?}
Parameter scaling e.g. in the frame head is associated with a multiplicative increase in walltime, further inflating what is already months of training time for 4000 steps.
Therefore no exhaustive sweeps were conducted as of now for a concrete scaling law.
However, we point to an experiment conducted on the ``compilers and interpreters'' dataset where a 12-layer framehead resulted in only a marginal improvement over a 6-layer framehead, while subsequently being outperformed by a 3-layer framehead with higher sequence length scaling\footnote{See Experiments \ref{sec:compilers-and-interpreters-experiments}}.
We believe that this is due to the fact that the overall effective depth of this network is already so high that increasing per-frame, non-hidden state dependent layer depth comes with significant diminishing returns.
Frame-head width scaling in conjunction with cell-capacity scaling appears to be the most effective way to improve convergence as a function of parameter count, resulting in early descent which is usually characteristic of continued improvements, however these runs were aborted early due to excessive step walltime requirements.
A promising middle ground seems to be scaling main sequence model cell capacity without significantly scaling frame head width, which was explored in an experiment on the ``GitHub technical excellence'' dataset.
And yet this model was again outperformed by a higher sequence length run with less than a quarter of its cell capacity\footnote{See Experiment \ref{sec:github-technical-excellence-experiments}}.

\subsection{How does this compare to vanilla Transformers?}
Direct comparison of frame-based action models to vanilla sequence models is difficult due to the fact that concatenative autoregression ``perceives'' differently from the more generalized autoregression present in frame-based action models.
Training directly on the targets of the same dataset would leave the model ``tripping in the dark'' as to e.g. current cursor state, and thus also position within the file of where the prediction occurs.
Additionally, given the model would e.g. lack knowledge of file contents of reopened files, it is unreasonable to expect the prediction objective to succeed in any meaningful way beyond memorizing fashion.
The best comparison we can make is the immediate next prediction of the frame-head, which is transformer-based---although still not a traditional sequence-to-sequence transformer.
Without the main sequence model---or generally at low frame sequence lengths, performance is unsatisfactory compared to their higher sequence length counterparts. We refer to experiments with sequence length 2, where the effect of the LSTM sequence model is minuscule.

\subsection{What causes the power law?}
It should be noted that while causal attention stagnates earlier than when using an LSTM as the main sequence model\footnote{See Appendix \ref{sec:non-linear-serial-integration-vs-weighting-based-aggregation} for comparison of serial integration vs. weighting-based aggregation.},
increasing sequence length increases performance regardless of sequence model type.
Naively this would suggest that simply attending to more information assists in modeling performance.
However, adjacent frames are largely similar with the exception of inserted characters, deleted characters and in rare cases UI layout changes.
We thus conduct an experiment where the input text is fully observable within one frame, where characters are only inserted and not deleted. This objective is effectively equivalent to traditional language modeling\footnote{See Appendix \ref{sec:fully-observable-frame-experiment}. for details for information on fully observable frame construction.}.
In such a case if we assume the embedding formed by a given frame head to be only a latent representation of the frame and the frame alone, then no information is gained at all by attending to previous frames,
as the current frame already contains all information from previous frames excluding the very last token.
This suggests that the notion of a frame embedding is slightly misleading and actually contains additional information to satisfy the main sequence model's information acquisition needs.

While the average loss across the sequence may decrease as sequence length increases in traditional language modeling setups, this does not usually translate into lower loss for earlier token positions---in fact, as sequence length increases, loss at earlier token positions sometimes even increases.
This indicates that training on longer sequence lengths for LLMs simply increases the confidence of the model in later predictions as more information becomes available.

However, for frame-based action models, we observe that both early token loss and late token loss are within the vicinity of the mean cross entropy, validating the suitability of the metric for modeling performance.
This behavior is not observed in traditional language models\footnote{See Appendix \ref{sec:representativity-of-mean-cross-entropy} for details on early token loss.}.
Thus, higher sequence length simply enables models to reach overall levels of lower loss not just for later token positions, but for the entirety of the sequence.

We hypothesize that the credit assignment results in optimization pressure being exerted on the frame embeddings to already contain information relevant to future predictions, which may have synergistic effects for the immediate next prediction.
Additionally, recurrent neural networks contain nonlinear activation functions between every time step, making them similar to deep feedforward networks.
The relation between LSTMs \citep{hochreiter1997long}, Highway Networks \citep{srivastava2015highwaynetworks} and ResNets \citep{he2015deepresiduallearningimage} is well known in the literature, as ResNets are modeled after the LSTM's Constant Error Carousel (CEC) \citep{Schmidhuber2025WhoInventedRes}.

\subsection{The Scaling Hypothesis}
Backpropagation through time has been largely avoided in recent literature due to lack of parallelizability.
However, due to the fact that there is ``No Free Lunch for Parallelism'' we think a whole category of capabilities might be reserved for
models that act in inherently serial fashion, as articulated by the Serial Scaling Hypothesis \citep{liu2025serialscalinghypothesis}.
We think this is a crucial scaling dimension to explore given data that has an inherent sequential bias, e.g. iteratively improving code and fixing bugs, as is captured in our dataset.

\subsection{Implications}
Naively these scaling laws are not properly exploitable given current hardware.
Current hardware is designed for embarrassingly parallel tasks while this regime defaults to a near worst case scenario where the workload is memory bound, inherently serial with small kernel launch bounds and requiring frequent CPU involvement.
Beyond optimization to reduce overhead and development of custom kernels to reduce unnecessary memory materialization, little can be done to improve overall walltime as a function of sequence length.
However, we still think these scaling trends are necessary to explore given the nature of the data and the promises of potential emergent long term planning capabilities as a result of deep credit assignment.

\subsection{Decentralized Training}
Given that the scaling trends observed exhibit a clear benefit to longer sequences and also always amortize as a function of wall time, the synchronization frequency of a sufficiently scaled model
is extremely low. In our experiments, step time already reached 20 minutes at a sequence length of 1024 with a relatively shallow frame-head. Additionally due to the compute to parameter ratio exhibited by recurrent models, the synchronization time itself tends to zero comparatively as sequence length increases.
Time scaling associated with more parameters makes parameter scaling less attractive given the multiplicative increase in wall-time to an already high wall-time.
Synchronizing the parameters of a rather small model over the internet in an infrequent manner thus represents an ideal use case for decentralized distributed training.

\subsection{Limitations}
While the scaling trends observed are promising, due to the long wall-time required to train these models, we consider this kind of model difficult to train in the current regime.
We believe we did not yet reach a point where any empirically observable interpretable effect manifests in generations that would categorically set it apart from traditional language models.
For any potential emergent planning capabilities, we would expect that a training sequence length of at least $10^5$ is required, which is currently intractable due to immature training infrastructure.
Additionally, we expect that further scaling of $n_{hidden}$ and $d_{model}$ are required at these sequence lengths, along with an increase in the number of layers in the main sequence model.
To make large scale training of these models feasible, significant effort is likely required to reduce overhead with a tailor-made model implementation of the model utilizing custom kernels and fine-grained manual memory management.

\section{Conclusion}

Long-horizon perception and control demand \emph{true} serial computation. We formalized this via (i) \emph{true depth}—the number of inherently sequential steps—and (ii) \emph{recurrence completeness}—the ability to realize general, non-associative recurrent updates. From these definitions we proved three impossibility results: any architecture with a parallelizable forward \emph{or} backward pass cannot be recurrence-complete; architectures with parallelizable (scan-like) input aggregation are likewise excluded; and, as a corollary, constant-depth Transformers, SSM families, and “parallelizable RNNs” are insufficient in the worst case. We further introduced \emph{input-length proportionality} and \emph{input aggregation criticality}, predicting a critical horizon beyond which non-recurrence-complete models fail to form correct state.

Empirically, diagnostics that force serial evaluation (FRJT, Withheld Maze) exhibit depth-dependent cliffs for time-parallel models, while a lightweight LSTM generalizes substantially farther. In a practical setting, our \emph{Recurrence-Complete Frame-based Action Model}—attention within frames, LSTM over time—trained on GitHub-derived text-video shows a clear power law in trained sequence length at fixed parameters; longer sequences uniformly improve early and late positions and ultimately amortize their linear wall-time cost. Taken together, the theory and results indicate that serial computation is not only necessary, but potentially beneficial to increase model expressivity.

\bibliographystyle{plainnat}
\bibliography{references}

\begin{thebibliography}{38}
\providecommand{\natexlab}[1]{#1}
\providecommand{\url}[1]{\texttt{#1}}
\expandafter\ifx\csname urlstyle\endcsname\relax
  \providecommand{\doi}[1]{doi: #1}\else
  \providecommand{\doi}{doi: \begingroup \urlstyle{rm}\Url}\fi

\bibitem[Bai et~al.(2018)Bai, Kolter, and Koltun]{bai2018empiricalevaluationgenericconvolutional}
Shaojie Bai, J.~Zico Kolter, and Vladlen Koltun.
\newblock An empirical evaluation of generic convolutional and recurrent networks for sequence modeling, 2018.
\newblock URL \url{https://arxiv.org/abs/1803.01271}.

\bibitem[Beck et~al.(2024)Beck, Pöppel, Spanring, Auer, Prudnikova, Kopp, Klambauer, Brandstetter, and Hochreiter]{beck2024xlstmextendedlongshortterm}
Maximilian Beck, Korbinian Pöppel, Markus Spanring, Andreas Auer, Oleksandra Prudnikova, Michael Kopp, Günter Klambauer, Johannes Brandstetter, and Sepp Hochreiter.
\newblock xlstm: Extended long short-term memory, 2024.
\newblock URL \url{https://arxiv.org/abs/2405.04517}.

\bibitem[Beltagy et~al.(2020)Beltagy, Peters, and Cohan]{beltagy2020longformerlongdocumenttransformer}
Iz~Beltagy, Matthew~E. Peters, and Arman Cohan.
\newblock Longformer: The long-document transformer, 2020.
\newblock URL \url{https://arxiv.org/abs/2004.05150}.

\bibitem[Choromanski et~al.(2020)Choromanski, Likhosherstov, Dohan, Song, Gane, Sarlos, Hawkins, Davis, Mohiuddin, Kaiser, Belanger, Colwell, and Weller]{choromanski2022rethinkingattentionperformers}
Krzysztof Choromanski, Valerii Likhosherstov, David Dohan, Xingyou Song, Andreea Gane, Tamas Sarlos, Peter Hawkins, Jared Davis, Afroz Mohiuddin, Lukasz Kaiser, David Belanger, Lucy Colwell, and Adrian Weller.
\newblock Rethinking attention with performers, 2020.
\newblock URL \url{https://arxiv.org/abs/2009.14794}.

\bibitem[Dao and Gu(2024)]{dao2024transformersssmsgeneralizedmodels}
Tri Dao and Albert Gu.
\newblock Transformers are ssms: Generalized models and efficient algorithms through structured state space duality, 2024.
\newblock URL \url{https://arxiv.org/abs/2405.21060}.

\bibitem[Dauphin et~al.(2017)Dauphin, Fan, Auli, and Grangier]{dauphin2017languagemodelinggatedconvolutional}
Yann~N. Dauphin, Angela Fan, Michael Auli, and David Grangier.
\newblock Language modeling with gated convolutional networks, 2017.
\newblock URL \url{https://arxiv.org/abs/1612.08083}.

\bibitem[Feng et~al.(2024)Feng, Tung, Ahmed, Bengio, and Hajimirsadeghi]{feng2024rnnsneeded}
Leo Feng, Frederick Tung, Mohamed~Osama Ahmed, Yoshua Bengio, and Hossein Hajimirsadeghi.
\newblock Were rnns all we needed?, 2024.
\newblock URL \url{https://arxiv.org/abs/2410.01201}.

\bibitem[Fu et~al.(2023{\natexlab{a}})Fu, Arora, Grogan, Johnson, Eyuboglu, Thomas, Spector, Poli, Rudra, and Ré]{fu2023monarchmixersimplesubquadratic}
Daniel~Y. Fu, Simran Arora, Jessica Grogan, Isys Johnson, Sabri Eyuboglu, Armin~W. Thomas, Benjamin Spector, Michael Poli, Atri Rudra, and Christopher Ré.
\newblock Monarch mixer: A simple sub-quadratic gemm-based architecture, 2023{\natexlab{a}}.
\newblock URL \url{https://arxiv.org/abs/2310.12109}.

\bibitem[Fu et~al.(2023{\natexlab{b}})Fu, Dao, Saab, Thomas, Rudra, and Ré]{fu2023hungryhungryhipposlanguage}
Daniel~Y. Fu, Tri Dao, Khaled~K. Saab, Armin~W. Thomas, Atri Rudra, and Christopher Ré.
\newblock Hungry hungry hippos: Towards language modeling with state space models, 2023{\natexlab{b}}.
\newblock URL \url{https://arxiv.org/abs/2212.14052}.

\bibitem[Fu et~al.(2023{\natexlab{c}})Fu, Kumbong, Nguyen, and Ré]{fu2023flashfftconvefficientconvolutionslong}
Daniel~Y. Fu, Hermann Kumbong, Eric Nguyen, and Christopher Ré.
\newblock Flashfftconv: Efficient convolutions for long sequences with tensor cores, 2023{\natexlab{c}}.
\newblock URL \url{https://arxiv.org/abs/2311.05908}.

\bibitem[Gers et~al.(1999)Gers, Schmidhuber, and Cummins]{818041}
F.A. Gers, J.~Schmidhuber, and F.~Cummins.
\newblock Learning to forget: continual prediction with lstm.
\newblock In \emph{1999 Ninth International Conference on Artificial Neural Networks ICANN 99. (Conf. Publ. No. 470)}, volume~2, pages 850--855 vol.2, 1999.
\newblock \doi{10.1049/cp:19991218}.

\bibitem[Gu and Dao(2024)]{gu2024mambalineartimesequencemodeling}
Albert Gu and Tri Dao.
\newblock Mamba: Linear-time sequence modeling with selective state spaces, 2024.
\newblock URL \url{https://arxiv.org/abs/2312.00752}.

\bibitem[Gu et~al.(2022)Gu, Goel, and Ré]{gu2022efficientlymodelinglongsequences}
Albert Gu, Karan Goel, and Christopher Ré.
\newblock Efficiently modeling long sequences with structured state spaces, 2022.
\newblock URL \url{https://arxiv.org/abs/2111.00396}.

\bibitem[Gupta et~al.(2022)Gupta, Gu, and Berant]{gupta2022diagonalstatespaceseffective}
Ankit Gupta, Albert Gu, and Jonathan Berant.
\newblock Diagonal state spaces are as effective as structured state spaces, 2022.
\newblock URL \url{https://arxiv.org/abs/2203.14343}.

\bibitem[He et~al.(2015)He, Zhang, Ren, and Sun]{he2015deepresiduallearningimage}
Kaiming He, Xiangyu Zhang, Shaoqing Ren, and Jian Sun.
\newblock Deep residual learning for image recognition, 2015.
\newblock URL \url{https://arxiv.org/abs/1512.03385}.

\bibitem[Hochreiter and Schmidhuber(1997)]{hochreiter1997long}
Sepp Hochreiter and Jürgen Schmidhuber.
\newblock Long short-term memory, 1997.
\newblock URL \url{https://www.researchgate.net/publication/13853244_Long_Short-term_Memory}.

\bibitem[Hotz and {Tinygrad Contributors}(2020)]{tinygrad}
George Hotz and {Tinygrad Contributors}.
\newblock Tinygrad: A minimalistic deep learning framework, 2020.
\newblock URL \url{https://tinygrad.org}.

\bibitem[Jordan and {Muon Contributors}(2024)]{jordan2024muon}
K.~Jordan and {Muon Contributors}.
\newblock Muon: Faster neural network training with momentum-corrected updates.
\newblock \url{https://kellerjordan.github.io/posts/muon/}, 2024.
\newblock Accessed October 2024.

\bibitem[Katharopoulos et~al.(2020)Katharopoulos, Vyas, Pappas, and Fleuret]{katharopoulos2020transformersrnnsfastautoregressive}
Angelos Katharopoulos, Apoorv Vyas, Nikolaos Pappas, and François Fleuret.
\newblock Transformers are rnns: Fast autoregressive transformers with linear attention, 2020.
\newblock URL \url{https://arxiv.org/abs/2006.16236}.

\bibitem[Kitaev et~al.(2020)Kitaev, Łukasz Kaiser, and Levskaya]{kitaev2020reformerefficienttransformer}
Nikita Kitaev, Łukasz Kaiser, and Anselm Levskaya.
\newblock Reformer: The efficient transformer, 2020.
\newblock URL \url{https://arxiv.org/abs/2001.04451}.

\bibitem[Li et~al.(2024)Li, Liu, Zhou, and Ma]{li2024chainthoughtempowerstransformers}
Zhiyuan Li, Hong Liu, Denny Zhou, and Tengyu Ma.
\newblock Chain of thought empowers transformers to solve inherently serial problems, 2024.
\newblock URL \url{https://arxiv.org/abs/2402.12875}.

\bibitem[Liu et~al.(2025)Liu, Preechakul, Kuwaranancharoen, and Bai]{liu2025serialscalinghypothesis}
Yuxi Liu, Konpat Preechakul, Kananart Kuwaranancharoen, and Yutong Bai.
\newblock The serial scaling hypothesis, 2025.
\newblock URL \url{https://arxiv.org/abs/2507.12549}.

\bibitem[Loshchilov and Hutter(2019)]{loshchilov2019decoupled}
Ilya Loshchilov and Frank Hutter.
\newblock Decoupled weight decay regularization.
\newblock In \emph{International Conference on Learning Representations}, 2019.
\newblock URL \url{https://openreview.net/forum?id=Bkg6RiCqY7}.

\bibitem[Merrill et~al.(2025)Merrill, Petty, and Sabharwal]{merrill2025illusionstatestatespacemodels}
William Merrill, Jackson Petty, and Ashish Sabharwal.
\newblock The illusion of state in state-space models, 2025.
\newblock URL \url{https://arxiv.org/abs/2404.08819}.

\bibitem[Peng et~al.(2023)Peng, Alcaide, Anthony, Albalak, Arcadinho, Biderman, Cao, Cheng, Chung, Grella, GV, He, Hou, Lin, Kazienko, Kocon, Kong, Koptyra, Lau, Mantri, Mom, Saito, Song, Tang, Wang, Wind, Wozniak, Zhang, Zhang, Zhao, Zhou, Zhou, Zhu, and Zhu]{peng2023rwkvreinventingrnnstransformer}
Bo~Peng, Eric Alcaide, Quentin Anthony, Alon Albalak, Samuel Arcadinho, Stella Biderman, Huanqi Cao, Xin Cheng, Michael Chung, Matteo Grella, Kranthi~Kiran GV, Xuzheng He, Haowen Hou, Jiaju Lin, Przemyslaw Kazienko, Jan Kocon, Jiaming Kong, Bartlomiej Koptyra, Hayden Lau, Krishna Sri~Ipsit Mantri, Ferdinand Mom, Atsushi Saito, Guangyu Song, Xiangru Tang, Bolun Wang, Johan~S. Wind, Stanislaw Wozniak, Ruichong Zhang, Zhenyuan Zhang, Qihang Zhao, Peng Zhou, Qinghua Zhou, Jian Zhu, and Rui-Jie Zhu.
\newblock Rwkv: Reinventing rnns for the transformer era, 2023.
\newblock URL \url{https://arxiv.org/abs/2305.13048}.

\bibitem[Poli et~al.(2023)Poli, Massaroli, Nguyen, Fu, Dao, Baccus, Bengio, Ermon, and Ré]{poli2023hyenahierarchylargerconvolutional}
Michael Poli, Stefano Massaroli, Eric Nguyen, Daniel~Y. Fu, Tri Dao, Stephen Baccus, Yoshua Bengio, Stefano Ermon, and Christopher Ré.
\newblock Hyena hierarchy: Towards larger convolutional language models, 2023.
\newblock URL \url{https://arxiv.org/abs/2302.10866}.

\bibitem[Schmidhuber(2025)]{Schmidhuber2025WhoInventedRes}
J{\"u}rgen Schmidhuber.
\newblock Who invented deep residual learning?
\newblock Technical Report IDSIA-09-25, IDSIA, 2025.
\newblock URL \url{https://people.idsia.ch/~juergen/who-invented-residual-neural-networks.html}.

\bibitem[Smith et~al.(2023)Smith, Warrington, and Linderman]{smith2023simplifiedstatespacelayers}
Jimmy T.~H. Smith, Andrew Warrington, and Scott~W. Linderman.
\newblock Simplified state space layers for sequence modeling, 2023.
\newblock URL \url{https://arxiv.org/abs/2208.04933}.

\bibitem[Srivastava et~al.(2015)Srivastava, Greff, and Schmidhuber]{srivastava2015highwaynetworks}
Rupesh~Kumar Srivastava, Klaus Greff, and Jürgen Schmidhuber.
\newblock Highway networks, 2015.
\newblock URL \url{https://arxiv.org/abs/1505.00387}.

\bibitem[Sun et~al.(2023)Sun, Dong, Huang, Ma, Xia, Xue, Wang, and Wei]{sun2023retentivenetworksuccessortransformer}
Yutao Sun, Li~Dong, Shaohan Huang, Shuming Ma, Yuqing Xia, Jilong Xue, Jianyong Wang, and Furu Wei.
\newblock Retentive network: A successor to transformer for large language models, 2023.
\newblock URL \url{https://arxiv.org/abs/2307.08621}.

\bibitem[Vaswani et~al.(2017)Vaswani, Shazeer, Parmar, Uszkoreit, Jones, Gomez, Kaiser, and Polosukhin]{vaswani2017attention}
Ashish Vaswani, Noam Shazeer, Niki Parmar, Jakob Uszkoreit, Llion Jones, Aidan~N. Gomez, Lukasz Kaiser, and Illia Polosukhin.
\newblock Attention is all you need, 2017.
\newblock URL \url{https://arxiv.org/abs/1706.03762}.

\bibitem[Wang et~al.(2020)Wang, Li, Khabsa, Fang, and Ma]{wang2020linformerselfattentionlinearcomplexity}
Sinong Wang, Belinda~Z. Li, Madian Khabsa, Han Fang, and Hao Ma.
\newblock Linformer: Self-attention with linear complexity, 2020.
\newblock URL \url{https://arxiv.org/abs/2006.04768}.

\bibitem[Williams and Zipser(1989)]{williams1989learning}
Ronald~J. Williams and David Zipser.
\newblock A learning algorithm for continually running fully recurrent neural networks.
\newblock \emph{Neural Computation}, 1\penalty0 (2):\penalty0 270--280, 1989.

\bibitem[Yang et~al.(2024{\natexlab{a}})Yang, Wang, Shen, Panda, and Kim]{yang2024gatedlinearattentiontransformers}
Songlin Yang, Bailin Wang, Yikang Shen, Rameswar Panda, and Yoon Kim.
\newblock Gated linear attention transformers with hardware-efficient training, 2024{\natexlab{a}}.
\newblock URL \url{https://arxiv.org/abs/2312.06635}.

\bibitem[Yang et~al.(2024{\natexlab{b}})Yang, Wang, Zhang, Shen, and Kim]{yang2024parallelizinglineartransformersdelta}
Songlin Yang, Bailin Wang, Yu~Zhang, Yikang Shen, and Yoon Kim.
\newblock Parallelizing linear transformers with the delta rule over sequence length, 2024{\natexlab{b}}.
\newblock URL \url{https://arxiv.org/abs/2406.06484}.

\bibitem[Yau et~al.(2025)Yau, Gupta, Engelmayer, Irie, Jegelka, and Andreas]{yau2025sequentialparalleldualityprefixscannable}
Morris Yau, Sharut Gupta, Valerie Engelmayer, Kazuki Irie, Stefanie Jegelka, and Jacob Andreas.
\newblock Sequential-parallel duality in prefix scannable models, 2025.
\newblock URL \url{https://arxiv.org/abs/2506.10918}.

\bibitem[Zaheer et~al.(2020)Zaheer, Guruganesh, Dubey, Ainslie, Alberti, Ontanon, Pham, Ravula, Wang, Yang, and Ahmed]{zaheer2021bigbirdtransformerslonger}
Manzil Zaheer, Guru Guruganesh, Avinava Dubey, Joshua Ainslie, Chris Alberti, Santiago Ontanon, Philip Pham, Anirudh Ravula, Qifan Wang, Li~Yang, and Amr Ahmed.
\newblock Big bird: Transformers for longer sequences, 2020.
\newblock URL \url{https://arxiv.org/abs/2007.14062}.

\bibitem[Zhang et~al.(2024)Zhang, Abdul-Mageed, and Lakshmanan]{zhang2024autoregressivechainthought}
Xiang Zhang, Muhammad Abdul-Mageed, and Laks V.~S. Lakshmanan.
\newblock Autoregressive + chain of thought = recurrent: Recurrence's role in language models' computability and a revisit of recurrent transformer, 2024.
\newblock URL \url{https://arxiv.org/abs/2409.09239}.

\end{thebibliography}

\appendix

\section{No Free Lunch for Parallelism Proof}
\label{sec:no-free-lunch-proof}
\paragraph{Model.}
A computation is a finite acyclic graph (DAG) of unit-cost primitive operations; edges are data dependencies.
The \emph{true depth} is the length of the longest directed path.
For a function \(y=F(h_t,z)\), we say \(y\) \emph{depends on} \(h_t\) if
\(\exists\,h_t,h_t',z:\ F(h_t,z)\neq F(h_t',z)\).

\begin{theorem}
\label{thm:rc-depth}
Assume the architecture is \emph{recurrence-complete}: for any function
\(g:\mathcal H^k\times\mathcal X\to\mathcal H\) there exists a program in the architecture computing
\(h_t = g(h_{t-1},\ldots,h_{t-k},x_t)\).
Regard \(g\) as an opaque primitive (no algebraic identities are assumed beyond extensional equality).
Then for each sequence length \(n\) there exists a choice of \(g\) and inputs
such that any correct computation has true depth \(\Omega(n)\).
\end{theorem}

\begin{proof}
  Let a computation be a DAG whose nodes are unit-cost primitive operations with edges denoting data dependencies.
  Assume the architecture can realize the update
  \(
  h_{t+1} = g(h_t,\ldots,h_{t-k},x_{t+1})
  \)
  for an \emph{opaque} $g$, and suppose for all $t$ the value $h_{t+1}$ \emph{depends on} $h_t$, i.e., there exist $h_t \neq h_t'$ with the same other inputs such that $g$ outputs different values.
  Consider any correct computation DAG producing $h_0,\ldots,h_n$.
  If $h_t$ were \emph{not} an ancestor of $h_{t+1}$ in the DAG, then $h_{t+1}$ would be a function of nodes that are independent of $h_t$, hence independent of the choice of $h_t$---contradicting the assumption that $h_{t+1}$ depends on $h_t$.
  Therefore $h_t$ is an ancestor of $h_{t+1}$ for every $t$.
  This yields a directed chain $h_0 \to h_1 \to \cdots \to h_n$ of length $n$, so the true depth is at least $n$.
\end{proof}

\begin{remark}
This is a worst-case bound over \(g\). Particular choices of \(g\) that obey associative/scannable
identities may admit \(O(\log n)\) depth, but recurrence-completeness requires representing
non-scanable \(g\) as well, and those force \(\Omega(n)\) serial steps.
\end{remark}

\section{Reverse-Mode non-parallelizability of Recurrence-Complete models}
\label{sec:reverse-depth-nonparallelizable}
\begin{lemma}
  \label{lem:reverse-depth}
  Under the setting of Theorem~\ref{thm:rc-depth}, assume $g$ is differentiable and for each $t$ the Jacobian $\partial g/\partial h_t$ depends on $h_t$. Let the loss be any function of $h_n$ with nonzero gradient. Then reverse-mode AD (backprop) has worst-case true depth $\Omega(n)$.
  \end{lemma}

  \begin{proof}
  Reverse-mode satisfies the recursion
  \(
  \lambda_t \;=\; \bigl(\partial g/\partial h_t\bigr)^\top \lambda_{t+1} \;+\; \cdots
  \)
  with $\lambda_n = \partial \mathcal{L}/\partial h_n$ and ``$\cdots$'' denoting terms that do not bypass $h_{t+1}$'s contribution.
  Thus $\lambda_{t}$ depends on $\lambda_{t+1}$ and (because $\partial g/\partial h_t$ depends on $h_t$) on $h_t$, which in turn depends on $h_{t-1}$, etc.
  By the same ancestor argument as in Theorem~\ref{thm:rc-depth}, $\lambda_{t+1}$ must be an ancestor of $\lambda_t$ for all $t$, yielding an $n$-long chain in the reverse pass.
  Hence the true depth is $\Omega(n)$.
\end{proof}

\section{Parallelizable Input-data aggregation precludes Recurrence-Completeness}
\label{sec:parallelizable-input-data-aggregation-precludes-recurrence-completeness}

We formalize ``parallelizable input aggregation'' as a \emph{structural} depth bound:
there exists a sublinear function $D(n)=o(n)$ such that for every parameter setting of the
architecture and every $t\le n$, the latent $\mathbf{h}_t$ can be computed from the prefix
$(\mathbf{x}_1,\dots,\mathbf{x}_t)$ by a circuit whose true depth is at most $D(t)$.
Equivalently: the forward computation that maps $(\mathbf{x}_1,\dots,\mathbf{x}_n)\mapsto \mathbf{h}_n$
has parameter-independent true depth $\le D(n)$.

\begin{theorem}
\label{thm:parallel-agg-not-rc-simple}
Any architecture whose input aggregation satisfies the above sublinear depth bound
is \emph{not} recurrence-complete.
\end{theorem}

\begin{proof}
Assume for contradiction that the architecture is recurrence-complete.
By Theorem~\ref{thm:rc-depth}, there exists a recurrent update $g$ and inputs such that any
correct computation producing $\mathbf{h}_n$ has true depth $\Omega(n)$.
But by the hypothesis of parallelizable aggregation, \emph{every} instantiation of the architecture
computes $\mathbf{h}_n$ with true depth at most $D(n)=o(n)$. For sufficiently large $n$ these
bounds are incompatible, yielding a contradiction. Hence the architecture cannot be
recurrence-complete.
\end{proof}

\begin{remark}
The same conclusion follows for reverse mode: Lemma~\ref{lem:reverse-depth} shows that a
recurrence-complete realization of such $g$ forces $\Omega(n)$ true depth in backprop.
If the architecture's aggregation/backward computation admits a sublinear depth bound,
this again contradicts recurrence-completeness.
\end{remark}

\section{Proof of Wall-time Amortization Claim}
\label{sec:walltime-amortization-proof}
\textbf{Claim.} For any $L_2>L_1$ there exists $T$ such that for all $t\ge T$ one has $\mathrm{loss}(t,L_2)<\mathrm{loss}(t,L_1)$.

\emph{Proof.} Let $s_i(t)=\gamma t/L_i$ for $i\in\{1,2\}$. Then
\begin{align}
\log\frac{\mathrm{loss}(t,L_2)}{\mathrm{loss}(t,L_1)}
&= \bigl(\log A(s_2)-\log A(s_1)\bigr)
    - \alpha(s_2)\log\frac{L_2}{L_1}
    + \bigl(\alpha(s_1)-\alpha(s_2)\bigr)\log L_1 \\
&\qquad + \log\!\bigl(1+r(L_2,s_2)\bigr) - \log\!\bigl(1+r(L_1,s_1)\bigr).
\label{eq:log-ratio}
\end{align}
As $t\to\infty$, $s_i(t)\to\infty$, so the first, third, and last two terms of \eqref{eq:log-ratio} vanish, while the middle term tends to $-\alpha_\infty\log(L_2/L_1)<0$. Hence the whole expression is eventually negative; exponentiating yields $\mathrm{loss}(t,L_2)<\mathrm{loss}(t,L_1)$ for all sufficiently large $t$.

Alternatively, for readers preferring explicit $\varepsilon$--$\delta$ bounds: fix $\varepsilon>0$ and choose $S$ so that for all $s\ge S$,
\[
\bigl|\log A(s)-\log A_\infty\bigr|\le\varepsilon,\quad
\bigl|\alpha(s)-\alpha_\infty\bigr|\le\varepsilon,\quad
\sup_{L\ge 1}\bigl|\log(1+r(L,s))\bigr|\le\varepsilon.
\]
Pick $T$ with $s_i(T)\ge S$ for $i=1,2$. Then for $t\ge T$,
\[
\log\frac{\mathrm{loss}(t,L_2)}{\mathrm{loss}(t,L_1)}
\le 3\varepsilon + \varepsilon\log L_1 - (\alpha_\infty-\varepsilon)\log\frac{L_2}{L_1},
\]
which is negative for sufficiently small $\varepsilon$, proving the claim.

\section{Non-linear, serial integration vs. weighting-based aggregation}
\label{sec:non-linear-serial-integration-vs-weighting-based-aggregation}
We compare serial integration of frame embeddings with parallel aggregation by replacing the main sequence model with causal self-attention and learned position embeddings.
In this setting, we observe earlier stagnation and diminishing returns as a function of training steps.
This trend can be seen in Figure \ref{fig:transformer-vs-lstm-seq}.

\begin{figure}[H]
  \centering
  \includegraphics[width=0.6\textwidth]{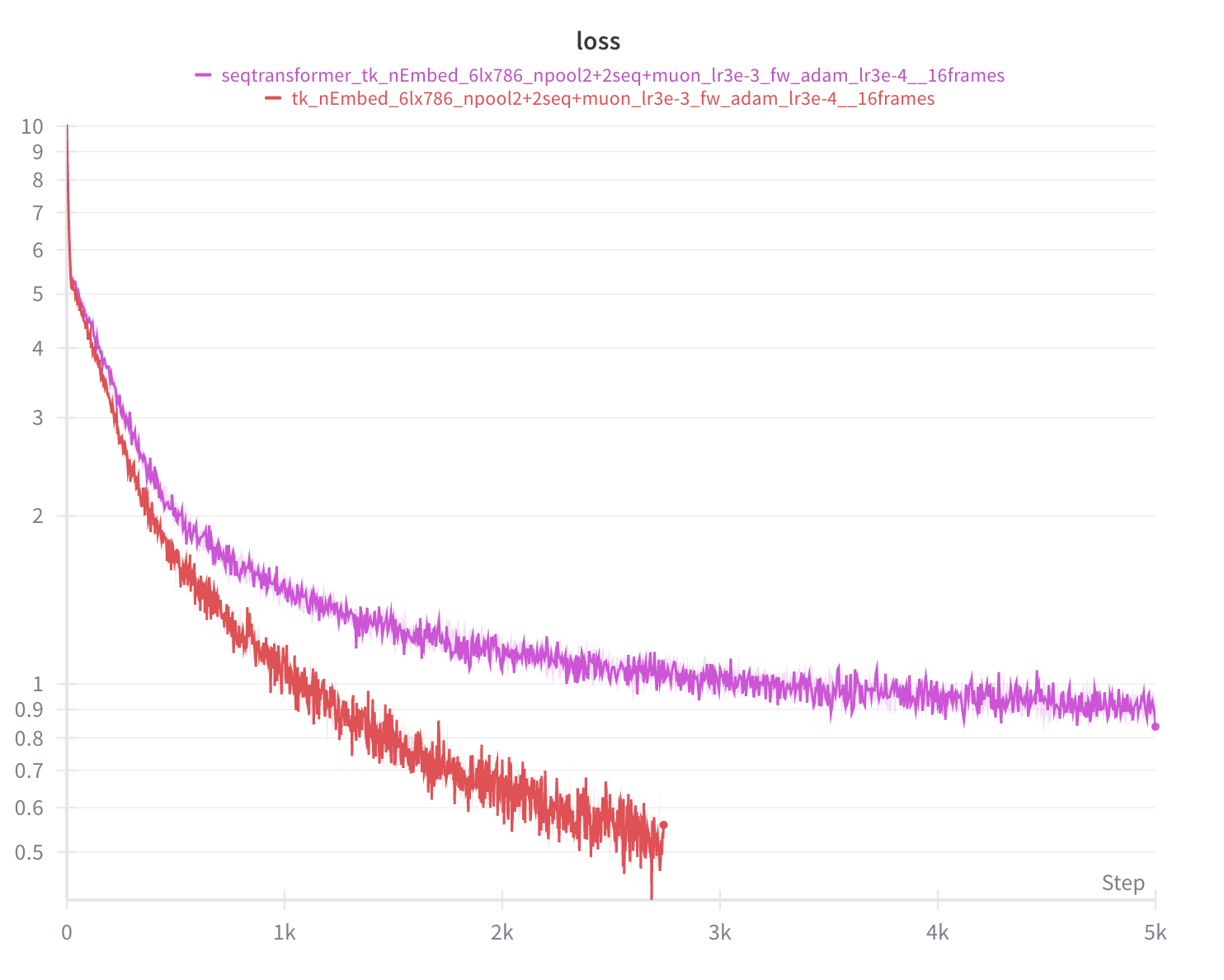}
  \caption{Transformer vs. LSTM as main sequence model}
  \label{fig:transformer-vs-lstm-seq}
\end{figure}
We note however that this experiment was conducted on a smaller dataset filtered for ``compilers and interpreters'' compared to the scaling laws described above, which were derived from a larger high-quality subset of GitHub filtered for ``technical excellence''.
The same type of power law however persists across both datasets and the experiment should retain significance.

\section{Critical Batch Size}
\label{sec:critical-batch-size}
We observe that there exists a critical batch size for a given sequence length after which increasing batch size does not significantly improve convergence speed as a function of step count.
Given that data-parallelism is the only horizontally scalable dimension for this architecture, we suggest maximizing the batch size is always desirable; however, we note that doing so does result in diminishing returns.

For example, scaling the batch size from 512 (dark green) to 1024 (orange) at a sequence length of 2 does allow doubling of the learning rate to $6\times 10^{-3}$, however training at sequence length 4 (light green) - matching the amount of supervised actions per update at batch size 512 - outperforms the training run drastically with lower learning rate.
\begin{figure}[H]
  \centering
  \includegraphics[width=0.6\textwidth]{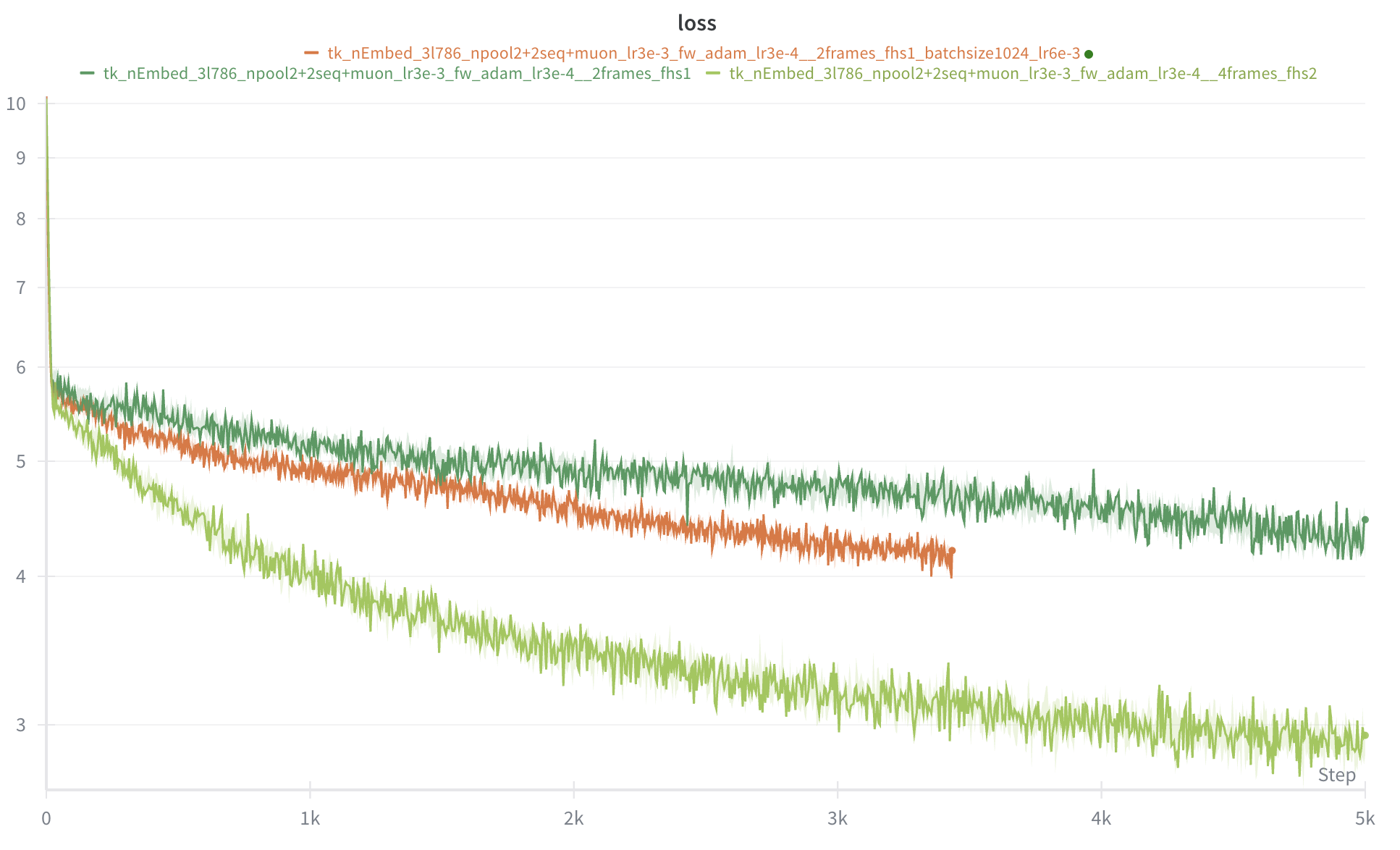}
  \caption{Sequence length 2 at batch size 512 \& 1024 vs. sequence length 4 at batch size 512}
  \label{fig:critical-batch-size}
\end{figure}

\section{Fully Observable Frame Experiment}
\label{sec:fully-observable-frame-experiment}

In this experiment, we train on a small toy-dataset where the predicted text is fully observable within one frame.
In the dataset, multiple Latin paragraphs are typed out while the file never exceeds the frame's bounds, therefore no scrolling ever occurs,
making the full state of the file observable in every frame.
In such a setting, if the frame embedding only logically contains the frame $x_t$ and no other information, then attending to previous frames is not useful,
as $x_{t-1}$ is already fully observable within $x_{t}$ due to the fact that characters are never deleted in this toy-dataset,
making successive frames merely additive compositions of the previous frame with the newly inserted characters.
The task is to overfit to the dataset, therefore we do not expect any generalization on unseen data.
However, the degree to which the model is able to correctly fit the dataset depends on sequence length.

\begin{figure}[H]
  \centering
  \includegraphics[width=0.8\textwidth]{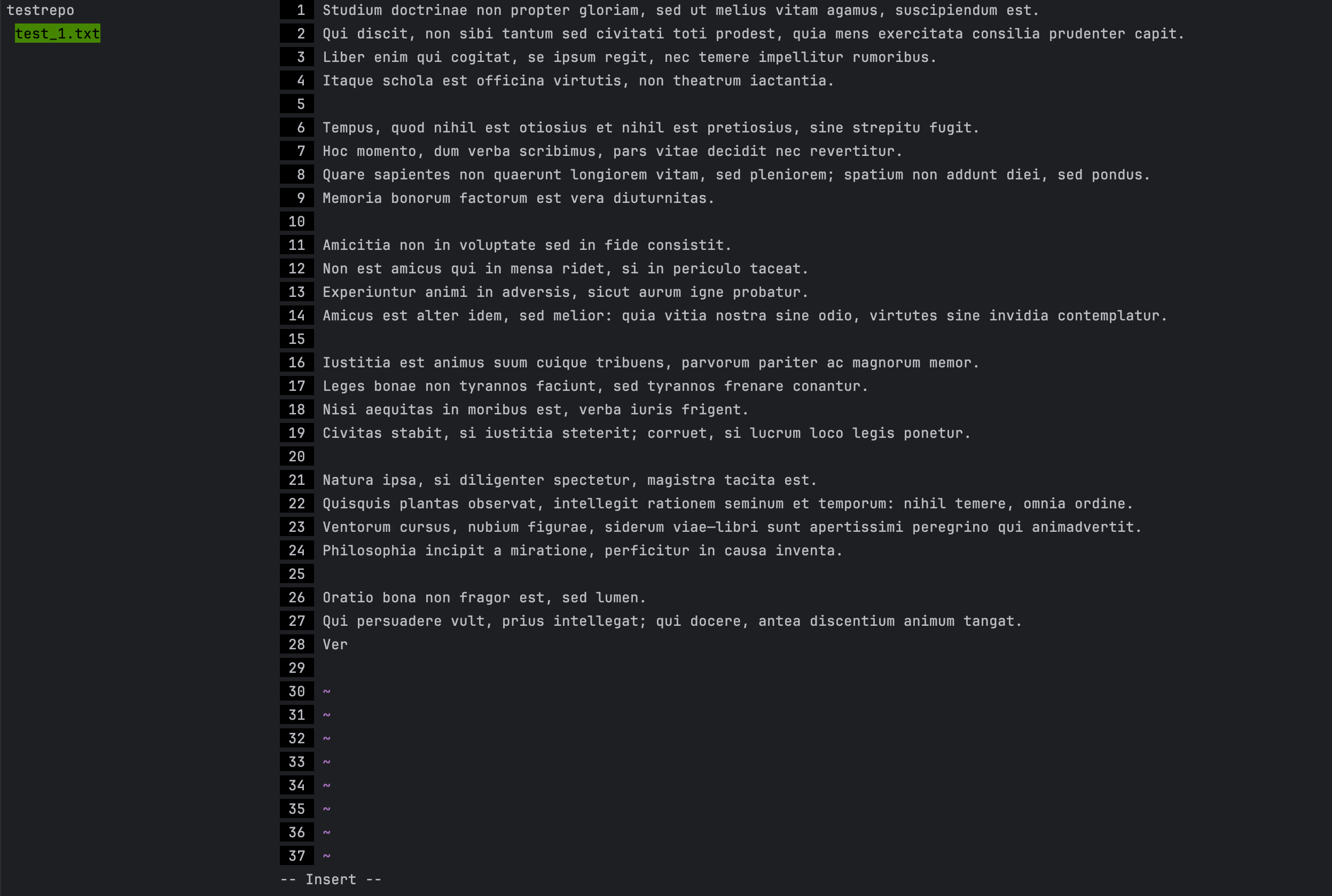}
  \caption{Fully Observable Frame Experiment}
  \label{fig:fully-observable-frame-experiment}
\end{figure}

Even in such a setting, the power law as a function of sequence length persists:

\begin{figure}[H]
  \centering
  \includegraphics[width=0.6\textwidth]{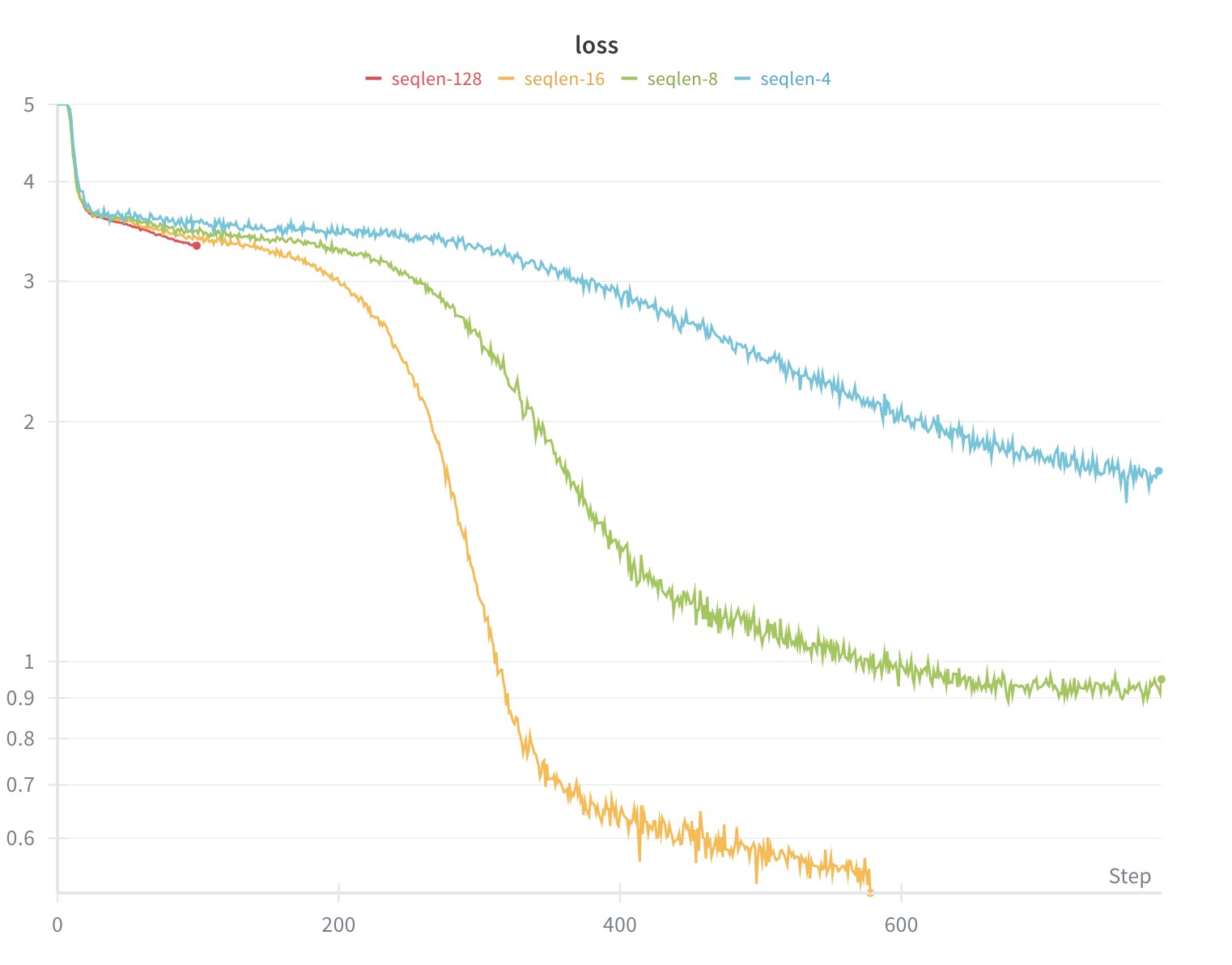}
  \caption{Fully Observable Frame Experiment Power Law}
  \label{fig:fully-observable-powerlaw}
\end{figure}

\section{Representativity of Mean Cross Entropy}
\label{sec:representativity-of-mean-cross-entropy}
For frame-based action models, we observe that both early token loss and late token loss are within the vicinity of the mean cross entropy.
This contradicts commonly accepted wisdom for traditional language models where training on longer sequence lengths may be detrimental to early token loss, resulting in an overall worse model when good performance in low token positions is desired.
In the fully observable frame experiment, we observe that increasing sequence length allows the model to reach lower loss overall and does so while maintaining mean loss in both low and high token positions within the vicinity of the mean cross entropy.
We train models for both sequence length 4 and 16 and batch size 512.
We measure validation cross-entropy loss, sustained accuracy, and evaluation accuracy across 32 tokens, forcing both models to length-generalize beyond their training sequence length.
Given that the task is explicitly to overfit to the dataset and to validate the extent to which the model is able to do so, data is thus sampled from the training set.

We observe that the mean cross-entropy across the first 4 and the last 4 tokens of the inferred sequence is within the vicinity of the mean validation cross-entropy of the full 32-token sequence.
Additionally, we observe that the validation mean cross-entropy is slightly lower than the training mean cross-entropy. We attribute this to a diluting effect of the difficult tokens present in the sequence, paired with graceful length generalization exhibited by the LSTMs,
beyond its training sequence length.
This solidifies average cross entropy as a representative metric for modeling performance and that its reduction as described per the scaling law is not caused due to skewed distribution across token positions.

\begin{figure}[H]
  \centering
  \includegraphics[width=0.6\textwidth]{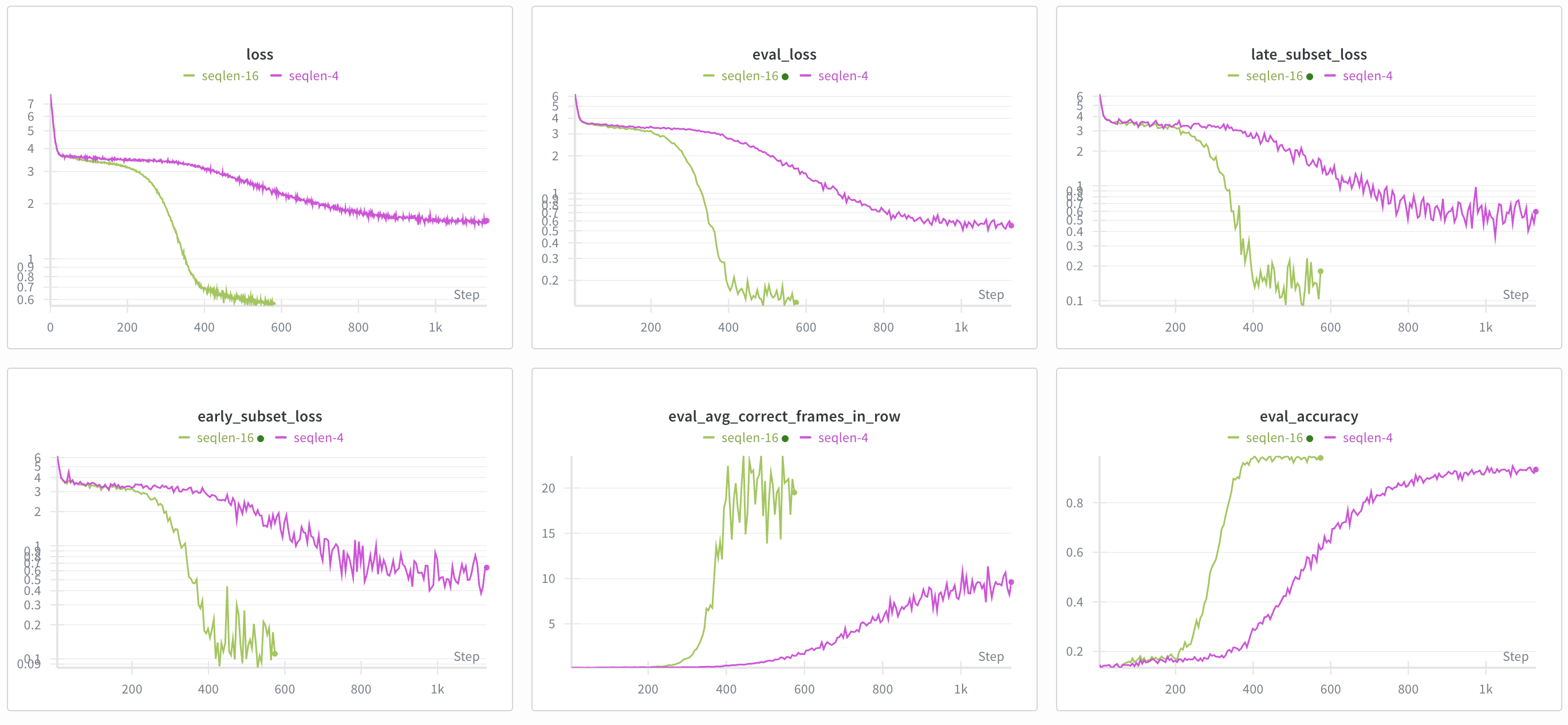}
  \caption{Validation metrics for fully observable frame experiment with sequence length 4 and 16}
  \label{fig:mean-cross-entropy}
\end{figure}

\section{Fixing Number of Actions per Update: Reducing Batch Size}
\label{sec:fixating-number-of-actions-per-update}
While we generally argue that for the regime of modeling long-running action sequences the fundamental unit of optimization should be the number of sequences---not the number of actions---
and we inherently acknowledge actions to be context-dependent and derivable from the past - hence the need to scale sequence length to increase confidence in future actions -
we also show that in the setting of fixing $batch\_size \times sequence\_length$, scaling sequence length while decreasing batch size still improves convergence speed as a function of step count and wall-time.
Doing so however inherently induces variance as sequence length increases, therefore we do not recommend reducing batch size in practice.
However, this increased variance only places higher sequence length runs at a disadvantage, which we still observe to perform better.

\begin{figure}[H]
  \centering
  \includegraphics[width=0.6\textwidth]{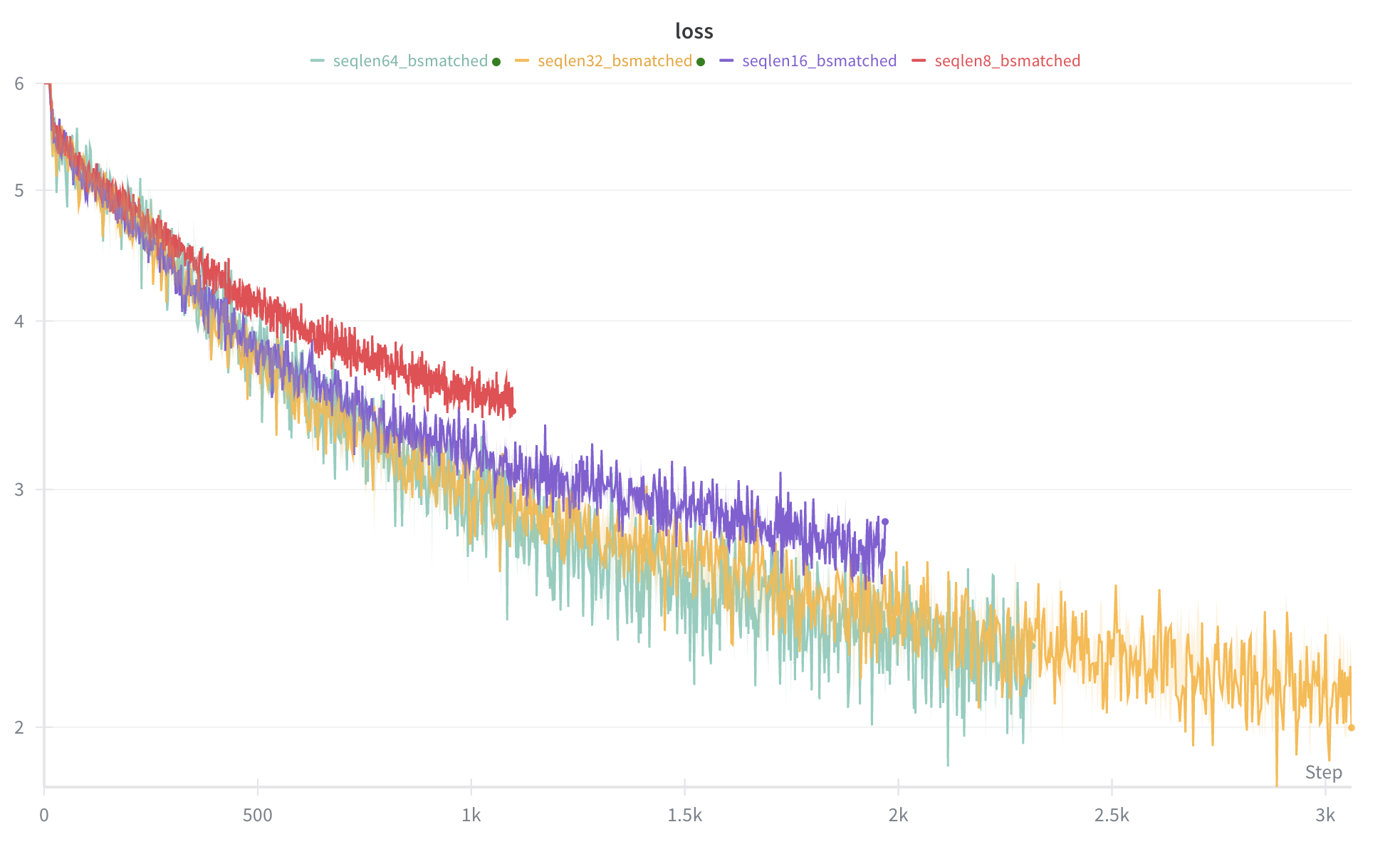}
  \caption{Fixing Number of Actions per Update}
  \label{fig:fixing-number-of-actions-per-update}
\end{figure}

\section{Model FLOP Estimation}
\label{sec:model-flops}
We estimate forward FLOPs assuming one multiply--add (MAC) counts as $2$ FLOPs and
neglecting bias/dropout/softmax/normalization and other elementwise costs (they are
$\mathcal{O}(BTD)$ and noted explicitly as ``low-order'').

\paragraph{Frame head.}
Notation: $B$ batch size, $D$ model width, $N_f$ tokens per frame before any pooling,
$P\in\mathbb{N}$ pooling stages (each halves the sequence length), $L_t$ transformer
blocks after the $P$ poolings. Let $T_s := N_f/2^s$ for $s=0,\dots,P$.

The $P$ pre-pooling blocks run at lengths $T_0,\dots,T_{P-1}$, then $L_t$ blocks run at
$T_P$, followed by a reduction LSTM of width $D$ over $T_P$.
Using a TF-block cost $24\,B\,T\,D^2 + 4\,B\,T^2 D$ and an LSTM cost $16\,B\,T\,D^2$,
\begin{align}
\mathrm{FLOPs}_{\text{frame-head}}
&\approx 24B D^2\sum_{s=0}^{P-1}T_s \;+\; 4B D\sum_{s=0}^{P-1}T_s^2
\;+\; L_t\!\left(24B D^2 T_P + 4B D T_P^2\right) \;+\; 16B D^2 T_P \;+\; \mathcal{O}(BDN_f)\\
&= B D N_f^2\!\left[\frac{16}{3}\!\left(1-4^{-P}\right) + 4 L_t\,4^{-P}\right]
\;+\; B D^2 N_f\!\left[48 + (24L_t-32)\,2^{-P}\right] \;+\; \mathcal{O}(BDN_f),
\end{align}
where we used $\sum_{s=0}^{P-1}T_s=2N_f(1-2^{-P})$ and $\sum_{s=0}^{P-1}T_s^2=\frac{4}{3}N_f^2(1-4^{-P})$.

\paragraph{Main sequence model.}
Notation: $T_s$ tokens in the sequence, $L_s$ layers, $H$ LSTM hidden size.
Per layer (LSTM then MLP $D\!\to\!4D\!\to\!D$):
\[
\mathrm{FLOPs}_{\text{layer}} \;\approx\; 8B T_s(DH+H^2) \;+\; 16B T_s D^2 \;+\; \mathcal{O}(B T_s D),
\]
hence
\[
\mathrm{FLOPs}_{\text{main}} \;\approx\; L_s\!\left[8B T_s(DH+H^2) + 16B T_s D^2\right] \;+\; \mathcal{O}(B T_s D L_s).
\]
In the common case $H=D$, this simplifies to $\mathrm{FLOPs}_{\text{main}} \approx 32B T_s D^2 L_s + \mathcal{O}(B T_s D L_s)$.

The following table shows the per-component forward FLOPs for often used hyperparameters.
\begin{figure}[H]
\begin{table}[H]
  \centering\small
  \begin{tabular}{lrr}
  \hline
  Component & FLOPs & Share \\
  \hline
  Frame-head TF block ($T=7680$) $\times 1$ & $1.4843\times 10^{14}$ & $51.60\%$ \\
  Frame-head TF block ($T=3840$) $\times 1$ & $5.1024\times 10^{13}$ & $17.74\%$ \\
  Frame-head TF blocks ($T=1920$) $\times 3$ & $5.9142\times 10^{13}$ & $20.56\%$ \\
  Frame-head reduction LSTM ($T=1920$) & $9.2771\times 10^{12}$ & $3.22\%$ \\
  Main LSTM (2 layers, $T=1024$) & $9.8956\times 10^{12}$ & $3.44\%$ \\
  Main MLP (2 layers, $T=1024$) & $9.8956\times 10^{12}$ & $3.44\%$ \\
  \hline
  \textbf{Frame-head subtotal} & $\mathbf{2.6788\times 10^{14}}$ & $\mathbf{93.12\%}$ \\
  \textbf{Main sequence subtotal} & $\mathbf{1.9791\times 10^{13}}$ & $\mathbf{6.88\%}$ \\
  \hline
  \textbf{Overall total} & $\mathbf{2.8767\times 10^{14}}$ & $\mathbf{100\%}$ \\
  \hline
  \end{tabular}
\end{table}
\caption{Per-component forward FLOPs (1 MAC = 2 FLOPs). Hyperparameters: $B=512$, $D=768$, frame size $48\times160$ ($N_f=7680$), pooling stages $P=2$ (each halves $T$), post-pooling transformer blocks $L_t=3$, main sequence length $T_s=1024$, depth $L_s=2$. Low-order terms (norms, residuals, pooling, activations) omitted.}
\end{figure}

\end{document}